\documentclass[runningheads]{llncs}

\usepackage{silence}
\usepackage{ellipsis, ragged2e}

\usepackage[english]{babel}
\usepackage[T1]{fontenc}
\usepackage[utf8]{inputenc}
\usepackage[final]{microtype}

\usepackage{amsmath,amssymb}
\usepackage{csquotes}
\usepackage{booktabs}
\usepackage[most]{tcolorbox}
\usepackage{multirow}
\usepackage{adjustbox}
\usepackage{caption}
\usepackage{subcaption}
\usepackage{hyperref}
\usepackage{rotating}

\usepackage{pgfplots}
\pgfplotsset{compat=1.18}
\usepackage{tikz}
\usetikzlibrary{shapes,positioning,arrows,arrows.meta,calc,automata,matrix,fit,backgrounds}
\tikzstyle{state}+=[minimum size = 6mm, inner sep=0,outer sep=1]
\usetikzlibrary{colorbrewer}
\usepgfplotslibrary{colorbrewer}

\colorlet{disabled}{lightgray}
\tikzstyle{state}=[draw,rectangle,inner sep=5pt,rounded corners=2pt]
\tikzstyle{action}=[font=\small,inner sep=0pt,outer sep=3pt]
\tikzstyle{actionnode}=[circle,draw=black,fill=black,minimum size=1mm,inner sep=0,outer sep=0]
\tikzstyle{actionedge}=[draw,-]
\tikzstyle{prob}=[font=\scriptsize,inner sep=0pt,outer sep=1pt]
\tikzstyle{probedge}=[draw,->]
\tikzstyle{directedge}=[draw,->]
\tikzset{chainarrow/.tip={Stealth[length=3pt]}}
\tikzset{>=chainarrow}

\usepackage{xparse}
\usepackage{mathtools}
\usepackage{environ}
\usepackage{marvosym}
\usepackage{bbm}
\usepackage{fontawesome5}
\usepackage{siunitx}

\usepackage{thmtools}
\usepackage{thm-restate}

\usepackage{algorithmicx,algorithm}
\usepackage[noend]{algpseudocode}

\usepackage{placeins} 

\usepackage[capitalize]{cleveref}
\Crefname{equation}{Eq.}{Eqs.}
\Crefname{figure}{Fig.}{Figs.}
\Crefname{tabular}{Tab.}{Tabs.}
\Crefname{remark}{Rem.}{Rems.}
\Crefname{section}{Sec.}{Secs.}
\Crefname{subsection}{Sec.}{Secs.}
\Crefname{proposition}{Prop.}{Props.}
\Crefname{example}{Ex.}{Exs.}

\usetikzlibrary{colorbrewer}
\usepackage{array}
\usepackage{todonotes}
\usepackage{mdframed}

\tcbset{every box/.style={
	left=1mm,right=1mm,top=1pt,bottom=1pt,boxrule=1pt,colbacktitle=white!0,colback=white!0,coltitle=black
}}

\colorlet{linecolor}{gray}
\newtcolorbox{linebox}[1]{
	empty,
	left skip=2mm,
	right skip=2mm,
	attach boxed title to top left,
	minipage boxed title,
	title=#1,
	boxed title style={empty,size=minimal,toprule=0pt,top=1mm,left=0mm,bottom=1mm,overlay={}},
	coltitle=black,fonttitle=\bfseries\scshape,
	before=\par\medskip\noindent,parbox=false,boxsep=0pt,left=1mm,right=0mm,top=2pt,breakable,pad at break=0mm,
	before upper=\csname @totalleftmargin\endcsname0pt,
}

\DeclarePairedDelimiter{\delimabs}{\lvert}{\rvert}
\DeclarePairedDelimiter{\delimcardinality}{\lvert}{\rvert}
\DeclarePairedDelimiter{\delimnorm}{\lVert}{\rVert}

\NewDocumentCommand{\abs}{sm}{\IfBooleanTF{#1}{\delimabs*{#2}}{\delimabs{#2}}}
\NewDocumentCommand{\cardinality}{sm}{\IfBooleanTF{#1}{\delimcardinality*{#2}}{\delimcardinality{#2}}}
\NewDocumentCommand{\norm}{sm}{\IfBooleanTF{#1}{\delimnorm*{#2}}{\delimnorm{#2}}}
\NewDocumentCommand{\powerset}{r()}{2^{#1}}

\newcommand{\indicator}[1]{\mathbbm{1}_{#1}}

\newcommand{\eqdef}{\vcentcolon=}

\newcommand{\unionSym}{\cup}
\newcommand{\unionBin}{\mathbin{\unionSym}}

\newcommand{\union}{\unionBin}

\newcommand{\Naturals}{\mathbb{N}}
\newcommand{\PosNaturals}{\mathbb{N}_+}

\newcommand{\Reals}{\mathbb{R}}

\NewDocumentCommand{\Measures}{d()}{\IfValueTF{#1}{\Pi(#1)}{\Pi}}
\NewDocumentCommand{\Distributions}{d()}{\IfValueTF{#1}{\mathcal{D}(#1)}{\mathcal{D}}}
\NewDocumentCommand{\integral}{d<> m m}{\IfValueTF{#1}{\int_{#1} #2\,d#3}{\int #2\,d#3}}
\NewDocumentCommand{\Expectation}{s d[]}{\IfValueTF{#2}{\mathbb{E}}{\mathbb{E}\IfBooleanTF{#1}{\left[#2\right]}{[#2]}}}
\NewDocumentCommand{\Probability}{s d[]}{\mathop{\mathrm{Pr}}\IfValueT{#2}{\IfBooleanTF{#1}{\left[#2\right]}{[#2]}}}
\NewDocumentCommand{\Variance}{s d[]}{\mathop{\mathrm{Var}}\IfValueT{#2}{\IfBooleanTF{#1}{\left[#2\right]}{[#2]}}}

\newcommand{\MDP}{\mathcal{M}}

\newcommand{\States}{S}
\newcommand{\initialstate}{{\hat{s}}}
\NewDocumentCommand{\mctransitions}{d()}{\IfValueTF{#1}{P(#1)}{P}}

\newcommand{\Actions}{A}
\NewDocumentCommand{\stateactions}{r()}{{\Actions}(#1)}
\NewDocumentCommand{\mdptransitions}{d()}{\IfNoValueTF{#1}{\mathsf{P}}{\mathsf{P}(#1)}}

\newcommand{\infinitepath}{\rho}

\NewDocumentCommand{\Infinitepaths}{d<>}{\IfValueTF{#1}{\mathsf{Paths}_{#1}}{\mathsf{Paths}}}
\NewDocumentCommand{\Finitepaths}{d<>}{\IfValueTF{#1}{\mathsf{FPaths}_{#1}}{\mathsf{FPaths}}}
\newcommand{\strategy}{\pi}
\NewDocumentCommand{\Strategies}{d<>}{\IfNoValueTF{#1}{\Pi}{\Pi_{#1}}}
\NewDocumentCommand{\StrategiesMD}{d<>}{\IfNoValueTF{#1}{\Pi}{\Pi_{#1}}^{\mathsf{MD}}}

\newcommand{\val}{\mathsf{V}}

\DeclareMathOperator{\SccsOp}{SCC}\NewDocumentCommand{\Sccs}{r()}{\SccsOp(#1)}
\DeclareMathOperator{\BsccsOp}{BSCC}\NewDocumentCommand{\Bsccs}{r()}{\BsccsOp(#1)}
\DeclareMathOperator{\EcsOp}{EC}\NewDocumentCommand{\Ecs}{d()}{\IfNoValueTF{#1}{\EcsOp}{\EcsOp(#1)}}
\DeclareMathOperator{\MecsOp}{MEC}\NewDocumentCommand{\Mecs}{d()}{\IfNoValueTF{#1}{\MecsOp}{\MecsOp(#1)}}

\NewDocumentCommand{\ProbabilityMC}{s r<> d[]}{\mathsf{Pr}_{#2}\IfNoValueF{#3}{\IfBooleanTF{#1}{\!\left[#3\right]\!}{[#3]}}}
\NewDocumentCommand{\ProbabilityMDP}{s r<> r<> d[]}{\mathsf{Pr}_{#2}^{#3}\IfNoValueF{#4}{\IfBooleanTF{#1}{\!\left[#4\right]\!}{[#4]}}}
\NewDocumentCommand{\ProbabilityMDPmax}{s r<> d[]}{\mathsf{Pr}_{#2}^{\max}\IfNoValueF{#3}{\IfBooleanTF{#1}{\!\left[#3\right]\!}{[#3]}}}
\NewDocumentCommand{\ProbabilityMDPsup}{s r<> d[]}{\mathsf{Pr}_{#2}^{\sup}\IfNoValueF{#3}{\IfBooleanTF{#1}{\!\left[#3\right]\!}{[#3]}}}

\NewDocumentCommand{\ExpectationMC}{s r<> r[]}{\mathbb{E}_{#2}\IfBooleanTF{#1}{\!\left[#3\right]\!}{[#3]}}

\newcommand{\Reachset}{G}
\NewDocumentCommand{\stepreach}{r<>}{\Diamond^{=#1}}
\NewDocumentCommand{\boundedreach}{r<>}{\Diamond^{{\leq}#1}}
\newcommand{\reach}{\Diamond}

\NewDocumentCommand{\steadystate}{d<> d()}{\IfValueTF{#1}{\pi^\infty_{#1}}{\pi^\infty}\IfValueT{#2}{(#2)}}

\usepackage{fontawesome5}

\newcommand{\stepsample}{\textsc{Sample}}
\newcommand{\stepinfer}{\textsc{Infer}}
\newcommand{\stepsolve}{\textsc{Solve}}

\newcommand{\confinter}{\mathcal{I}}

\newcommand{\confset}{\mathcal{C}}
\newcommand{\confseq}{\mathcal{S}}

\newcommand{\successes}{\xi}
\newcommand{\average}{\mu}
\newcommand{\variance}{\sigma}

\newcommand{\error}{\delta}

\newcommand{\KL}{\mathrm{KL}}

\newcommand{\expect}{\mathbb{E}}

\newcommand{\saprob}{P}
\newcommand{\sampletime}{t}
\newcommand{\successor}{r}
\newcommand{\Successors}{R}

\DeclareMathOperator*{\argmin}{arg\,min}
\title{Confidence Sequences for Online Statistical Model Checking of Markov Decision Processes\vspace{-.5em}}
\titlerunning{Confidence Sequences for Online MDP-SMC}
\author{
    Konstantin Kueffner\inst{1}
    \and
    Tobias Meggendorfer\inst{2}
	\and
	Maximilian Weininger\inst{3}
	\and
	Patrick Wienh\"oft\inst{4}
}

\authorrunning{K.~Kueffner et al.}

\institute{
	Institute of Science and Technology Austria, Klosterneuburg, Austria\\
    \and
    Lancaster University Leipzig, Leipzig, Germany\\
	\and
    Ruhr-University Bochum, Bochum, Germany\\
    \and
	TUD Dresden University of Technology, Dresden, Germany\\
}

\AtBeginDocument{
\setlength{\abovedisplayskip}{0.75\abovedisplayskip}
\setlength{\abovedisplayshortskip}{0.75\abovedisplayshortskip}
\setlength{\belowdisplayskip}{0.75\belowdisplayskip}
\setlength{\belowdisplayshortskip}{0.75\belowdisplayshortskip}
}

\advance\textheight2mm
\advance\voffset-1mm

\let\subparagraph\relax
\usepackage{titlesec}
\titlespacing*{\section}{0pt}{3ex plus 1ex minus 0ex}{2ex plus 1ex minus 0ex}
\titlespacing*{\subsection}{0pt}{1.5ex plus 0.5ex}{1ex plus 0.5ex}
\titlespacing{\subsubsection}{0pt}{1ex plus 0.5ex}{0.5ex plus 0.5ex}
\titlespacing{\paragraph}{0pt}{1ex plus 1ex}{1ex plus 0.5ex}

\crefname{section}{Sec.}{Secs.}%
\crefname{appendix}{App.}{Apps.}%
\crefname{lemma}{Lem.}{Lemms.}
\crefname{theorem}{Thm.}{Thms.}
\crefname{corollary}{Cor.}{Cors.}
\crefname{equation}{Eq.}{Eqs.}
\crefname{figure}{Fig.}{Figs.}
\crefname{table}{Tab.}{Tabs.}
\crefname{assumption}{Asm.}{Asms.}

\Crefname{section}{Sec.}{Secs.}%
\Crefname{appendix}{App.}{Apps.}%
\Crefname{lemma}{Lem.}{Lemms.}
\Crefname{theorem}{Thm.}{Thms.}
\Crefname{corollary}{Cor.}{Cors.}
\Crefname{equation}{Eq.}{Eqs.}
\Crefname{figure}{Fig.}{Figs.}
\Crefname{table}{Tab.}{Tabs.}
\Crefname{assumption}{Asm.}{Asms.}

\pgfplotsset{
    EvalList/.append style={
        cycle list/Paired,
        cycle multiindex* list={
            Paired\nextlist
            dashed,dashed,dashed,solid,solid,solid,solid,solid,solid,solid,solid,solid,solid,dotted\nextlist
        }, 
    }
}

\begin{document}

\maketitle\vspace{-1em}
\begin{abstract}
Markov decision processes (MDPs) are a classic model of decision making under uncertainty, exhibiting both non-deterministic choice as well as probabilistic uncertainty.
Traditionally, exact knowledge of the underlying probabilities is assumed.
However, this often is unrealistic, e.g.\ when modelling cyber-physical systems or biological processes.
Here, statistical methods provide a way towards obtaining meaningful guarantees.
The classical approach is to gather samples in the MDP, use these to draw statistical conclusions about the transition probabilities, and from there obtain bounds on the true value; then, if these bounds are too broad, repeat.
However, existing implementations of this approach are either subtly incorrect or sub-optimal, and quite often both.
We present several \emph{confidence sequences}, which are specifically designed for such \enquote{online} settings, implement all of them in an efficient tool, and show their practical applicability.
In particular, we show that they outperform classical \enquote{union-bound} style approaches, and overall our implementation requires 50x less samples on average than previous state of the art.
\end{abstract}

\section{Introduction}

\emph{Markov decision processes} (MDPs) \cite{Puterman} are \emph{the} classic modelling formalism for dynamic systems with probabilistic and nondeterministic behaviour.
Taking one of the available actions in each state moves the system to a successor state, which is chosen from a probability distribution associated with the state-action pair.
This is \emph{aleatory} uncertainty, i.e.\ the inherent randomness of the process, e.g.\ a coin toss.
However, there may also be \emph{epistemic} uncertainty (lack of knowledge), e.g.\ not knowing the bias of the coin, see \cite{DBLP:journals/sttt/BadingsSSJ23}.
While verification procedures traditionally assume full knowledge of the MDP (also called \enquote{white box}, with no epistemic uncertainty), over the past decade there has been an increasing interest in reliable analysis of MDPs where the exact transition probabilities are unknown (grey box) or even the presence of the transitions (positivity of the transition probability) is unknown (black box).
For example, consider the task of finding the optimal probability of getting heads from a coin toss when choosing among multiple coins.
Given knowledge of the coins' biases (white box), we clearly choose the coin with the highest biases towards heads.
This does not guarantee winning (there still is aleatoric uncertainty), but we know which coin to pick and the resulting value (the chances we get). 
With unknown biases (grey box), we cannot with certainty determine the value nor choose the best coin.

\emph{Statistical model checking} (SMC) \cite{DBLP:conf/cav/YounesS02} tackles this problem by
gathering simulations of the system (e.g.\ tossing each coin several times) and drawing statistical conclusions from the obtained data.
One of the strongest guarantees possible are \emph{probably approximately correct} (PAC) results, where the reported value is $\varepsilon$-close to the true result with confidence at least $\gamma$.
When considering discounted or finite-horizon objectives, the distance between estimation and true value can be upper-bounded for paths of a certain length.
Thus, PAC-guarantees and even bounds on sample-complexity can be derived in a straightforward manner, see e.g.~\cite{DBLP:conf/colt/AgarwalKY20,DBLP:conf/icml/JinKSY20,DBLP:conf/aaai/HasanzadeZonuzy21,DBLP:conf/icml/Shi0W0C22,DBLP:conf/colt/WagenmakerSJ22,DBLP:journals/tac/WenT22}.
In contrast, for infinite-horizon objectives, more involved methods are required.
There, the prevalent approach to obtain PAC guarantees (in the grey box setting)~\cite{AKW19,WeiningerGMK21,agarwal2022pac,DBLP:conf/nips/SuilenS0022,BaiDubWie23,DBLP:journals/jair/BadingsRAPPSJ23,DBLP:journals/corr/abs-2310-12248} is to repeat the following three steps: 
(i)~\stepsample{} the unknown MDP, counting occurrences of every transition, (ii)~\stepinfer{} an \emph{interval MDP} (IMDP) \cite{givan2000bounded} by obtaining confidence intervals of the transition probabilities,
and finally (iii)~\stepsolve{} the IMDP, obtaining lower and upper bounds $a, b$,
that, with probability $\gamma$, bound the true value of the unknown MDP, i.e.\ $v \in [a, b]$.
If additionally $|b - a| < \varepsilon$, the algorithm terminates; otherwise, it repeats the steps, i.e.\ gathers more data, increasing the precision of the confidence intervals in the IMDP, until $a$ and $b$ are $\varepsilon$-close.

However, existing variants rely on \emph{flawed correctness arguments} or are \emph{too conservative}, and quite often both.
Firstly, naively repeating the three steps compromises the PAC-guarantee, as the entire confidence $\gamma$ is \enquote{reused} across multiple \stepinfer{}s.
Intuitively, this is akin to repeating a statistical experiment until the result \enquote{looks right}.
(In \Cref{app:sota_bad} we provide a general counterexample based on the law of iterated logarithm~\cite{khintchine1924satz,grimmett2020probability}.)
Secondly, the used statistical methods are overly conservative:
On the one hand, the statistical methods used in \stepinfer{} are suboptimal, providing wide confidence intervals, in turn requiring more samples, see~\cite{WatO}.
However, even the \enquote{optimal} methods of \cite{WatO} are sub-optimal when we can repeatedly interact with the system.
Intuitively, estimates for the same distribution in different \stepinfer{} steps do not exploit the fact that the samples they are working with actually stem from the same distribution.
So, on the other hand, strictly dividing the process into these three phases is fundamentally problematic.

\subsubsection*{Contributions.}
We improve SMC with PAC guarantees for MDPs by employing \emph{confidence sequences}~\cite{darling1967confidence}.
For a stream of samples, these yield a stream of estimates with a statistical guarantee for the \emph{entire} (arbitrarily long) stream.
Using such sequentially correct estimates, we address the flawed correctness proof of previous works.
Moreover, confidence sequences exploit the fact that estimates at different times are \emph{not independent}, also replacing previous overly conservative estimates and significantly reducing the required number of samples.

Technically, we analyse \emph{gridded} variants of classical confidence intervals such as Hoeffding~\cite{Hoe63}, (empirical) Bernstein~\cite{MauPon09}, and Clopper-Pearson~\cite{CloPea34}.
Moreover, we adapt several confidence sequences, based on stitching~\cite{howard2021time}, betting~\cite{waudby2024estimating}, and test-martingales~\cite{waudby2024estimating}, to our setting.
Finally, we introduce a novel value-based hypothesis test for SMC, which entirely circumvents estimating probabilities.
For all these methods, we discuss correctness, convergence, and computability.
Practically, we provide an efficient, extensible Java tool embedding the confidence sequences not only in the state-of-the-art \stepsample{}-\stepinfer{}-\stepsolve{} pipeline, but also a new interwoven approach.
Our new variant exploits the continuous guarantees of confidence sequences to avoid constructing an IMDP for the whole model and instead locally applies the robust value iteration of~\cite{DBLP:conf/aaai/MeggendorferWW25}.
This enables soundly incorporating guidance heuristics for sampling~\cite{DBLP:journals/theoretics/BrazdilCCFKKM0U25} and utilizing the novel value-based hypothesis test.

Our evaluation shows a significant reduction in the number of samples required, on average 50x less than state-of-the-art, and that all of our contributions (confidence sequences, in particular the value test, and our interwoven approach) are instrumental to this improvement.
Additionally, our evaluation provides, to our knowledge for the first time, an empirical comparison of all presented statistical methods, both on single distributions and on SMC of MDPs.
This additionally demonstrates the benefits of confidence sequences and their individual strengths and weaknesses, also beyond our focus of SMC for MDPs.

\subsection*{Related Work}
\Cref{subsec:existing_methods} discusses other works on SMC of MDPs. \Cref{app:assumptions} details common SMC assumptions and how they reduce to the grey box setting we consider.

\paragraph{Other Formalisms for System and Specification.}
Continuous time systems can be tackled by estimating transitions rates from samples~\cite{agarwal2022pac,DBLP:conf/qest/BacciILR23}. 
Continuous space systems can be approached by abstraction to finite-state MDPs~\cite{DBLP:journals/jair/BadingsRAPPSJ23}; the incurred abstraction error can be transparently incorporated in the underlying robust value iteration~\cite{DBLP:conf/aaai/MeggendorferWW25} by expanding the uncertainty set.
For objectives, we focus on reachability specifications for simplicity but again can directly extend others as long as robust value iteration supports them:
Technically, our work provides statistically sound bounds on the Bellman backup used in robust value iteration.

\paragraph{SMC for Markov Chains.}
In Markov chains, the motivation for SMC is fundamentally different: The system usually is assumed to be known, but too large to handle, see e.g.\ \cite{DBLP:conf/cav/YounesS02,WitS}.
Since there is no non-determinism in Markov chains, one can, intuitively, just repeatedly sample entire paths from the system to gain a statistical estimate on the probability of reaching the goal.
In many cases, this is significantly faster than explicit computation, yielding a precision-performance trade-off (faster results but only with PAC guarantees), in particular since the exact model description is known, and successor states can be sampled very efficiently.
In contrast, for MDPs the main motivation is an inherent lack of knowledge of transition dynamics, and the underlying systems are usually very costly to simulate, shifting the focus from speed to reducing sample count.

\paragraph{Other Forms of Guarantees.}
When obtaining samples according to a behavioural policy, one can guarantee an improvement over this policy~\cite{DBLP:conf/ijcai/WienhoftSSDB023}.
In our setting, we explicitly explore unsafe states, while in a \enquote{safe online} setting, safety has to be guaranteed already while exploring the system.
In this direction, we refer to works on shielding~\cite{DBLP:conf/aaai/AlshiekhBEKNT18}, PAC online learning~\cite{concur18,uai20} and regret minimization~\cite{DBLP:conf/colt/WagenmakerSJ22}.

\section{Preliminaries}

In this section, we summarize all relevant theory necessary to explain our contributions.
Since our work touches upon many different topics, we keep this section brief and refer the reader to cited literature for further information.
We assume familiarity with the basics of probability theory 
(see e.g.\ \cite{billingsley2017probability}) and statistics (see e.g.\ \cite{wainwright2019high}).
A \emph{probability distribution} over a countable set $X$ is a mapping $d : X \to [0,1]$, such that $\sum_{x \in X} d(x) = 1$.
The set of all probability distributions over $X$ is $\Distributions(X)$.
For a set $\States$, we denote by $\States^*$ ($\States^\omega$) the set of all (in)finite sequences over $\States$.

\subsection{Markov Decision Processes}

A \emph{Markov decision process (MDP)}, e.g.\ \cite{Puterman}, is a tuple $\MDP = (\States, \Actions, \mdptransitions)$, where
$\States$ is a finite set of states;
$\Actions$ is a finite set of actions, overloaded to yield for each state $s \in \States$ a non-empty set of \emph{available actions} $\stateactions(s) \subseteq \Actions$; 
and $\mdptransitions : \States \times \Actions \rightharpoonup \Distributions(\States)$ is the (partial) transition function, that yields for each state $s \in \States$ and $a \in \stateactions(s)$ the associated distribution over successor states $\mdptransitions(s, a)$.
For simplicity, we write $\mdptransitions(s, a, \successor)$ instead of $\mdptransitions(s, a)(\successor)$.

The \emph{semantics} of MDPs are defined as usual by means of paths, strategies, and the probability measure in the induced Markov chain, see ~\cite[Chp.~10]{DBLP:books/daglib/0020348} for an extensive introduction.
An \emph{infinite path} is a sequence of state-action pairs $\infinitepath = s_1 a_1 s_2 a_2 \cdots \in (\States \times \Actions)^\omega$ with $\mdptransitions(s_i, a_i, s_{i+1}) > 0$.
We denote by $\infinitepath(i)$ 
the $i$-th state $s_i$ in a path $\infinitepath$ and by $\Infinitepaths<\MDP>$ the set of all infinite paths.
A (memoryless deterministic, MD) \emph{strategy} is a mapping $\strategy : \States \to \Actions$, choosing one enabled action in each state, i.e.\ $\strategy(s) \in \Actions(s)$.
(For our purposes, MD strategies are sufficient.)
We write $\StrategiesMD<\MDP>$ to refer to all MD strategies.
Complementing an MDP $\MDP$ with such a strategy yields a Markov chain, which together with an initial state $\initialstate \in \States$ induces a unique probability measure $\ProbabilityMDP<\MDP, \initialstate><\strategy>$ over infinite paths \cite[Chp.~10.1]{DBLP:books/daglib/0020348}.

An \emph{objective} formalises the goal of the MDP.
For simplicity, we focus on \emph{reachability}, but our methods easily generalize to other relevant objectives such as total reward, too.
A reachability objective concerns the probability to reach a given set of \emph{goal states} $\Reachset \subseteq \States$.
Formally, denote the set of paths that eventually reach $\Reachset$ as $\reach \Reachset \coloneqq \{\infinitepath \in \Infinitepaths<\MDP> \mid \exists i.~\infinitepath(i) \in \Reachset\}$.
The \emph{value} of a state is the maximum probability to achieve the objective, i.e.\ reach the goal states, under any strategy, formally $\val_{\MDP, \reach \Reachset}(s) \eqdef \max_{\strategy \in \StrategiesMD<\MDP>} \ProbabilityMDP<\MDP, s><\strategy>[\reach \Reachset]$.
We omit the subscripts of $\val$ where clear from the context.
MD strategies are sufficient for optimal reachability~\cite[Lem.\ 10.102]{DBLP:books/daglib/0020348}.
We are interested in approximating the value in a given initial state $\initialstate$, i.e.\ compute $v$ such that $|\val_{\MDP}(\initialstate) - v| \leq \varepsilon$ for a given $\varepsilon > 0$.

A classical approach to solve this problem is \emph{value iteration} (VI), the method of choice in most modern tools due to empirical scalability \cite{DBLP:conf/tacas/HartmannsJQW23}.
Starting from an initial value vector $L_0(s) = 1$ if $s \in \Reachset$ and $0$ otherwise, VI iterates the Bellman backup $\mathcal{B}(s, a, V) = \sum_{\successor \in \States} \mdptransitions(s, a, \successor) \cdot V(\successor)$, setting $L_{i+1}(s) = \max_{a \in \stateactions(s)} \mathcal{B}(s, a, L_i)$. 
Intuitively, $L_i(s)$ describes the maximal probability to reach $\Reachset$ from $s$ in $i$ steps.
$L_i$ converges monotonically to the correct value, i.e.\ $L_i(s) \leq L_{i+1}(s) \leq \val(s)$ for all $i$ and $\lim_{i \to \infty} L_i(s) = \val(s)$ \cite{CH08}.
Through a similar, but technically more involved process, we also obtain converging upper bounds $U_i$ (see e.g.\ \cite{HM18,DBLP:conf/lics/KretinskyMW23}), allowing to stop once $U_i(\initialstate) - L_i(\initialstate) \leq \varepsilon$.

\subsection{Statistical Model Checking}
\label{subsec:statistical_model_cheking}
In the classical setting, the entire MDP is provided as input.
In this work, we deal with MDPs where the exact transition probabilities are unknown. 
Instead, we have \emph{sampling access} to the system:
Starting in the initial state, we can choose an available action $a \in \stateactions(s)$, take a step according to the probability distribution $\mdptransitions(s, a)$ associated with $a$, observe the successor state $\successor$ we end up in, and repeat the process.
Moreover, we can reset the system into its initial state, i.e.\ we can sample (finite) paths in the MDP.
Repeatedly interacting with the system in this way allows us to draw \emph{statistical} conclusions.
For example, consider an MDP modelling a simple coin toss with unknown bias.
By tossing several times, we can get a \emph{statistical} estimate on its bias:
Tossing a coin \num{1000} times and seeing \num{400} heads, we can conclude that the actual bias of the coin likely lies in the \emph{uncertainty set} $[0.3, 0.5]$.
(We discuss the details on obtaining such estimates later.)
However, the true bias might actually be different and we just observed an unlikely outcome, hence we can never be certain.
Thus, we can only aim for a \emph{probably approximately correct} (PAC) result, i.e.\ given a \emph{confidence requirement} $\gamma$ and \emph{precision requirement} $\varepsilon$, we want a value $v$ such that $|\val(\initialstate) - v| \leq \varepsilon$ with confidence $\gamma$.
For readability, we write $\error = 1 - \gamma$ for the allowed error probability.

By applying statistical methods to each action of the MDP separately (e.g.\ using the union bound), we can obtain estimates of their behaviour (e.g.\ bounds on the transition probabilities) \cite{WatO}, and in turn compute bounds on the value using \emph{robust} value iteration \cite{givan2000bounded,DBLP:conf/aaai/MeggendorferWW25}.
The key ingredient this method requires is obtaining lower bounds on the Bellman backup, i.e.\ an operator $\mathfrak{L}$ which for any value function $V : \States \to \Reals$ and state-action pair $s \in \States$, $a \in \stateactions(s)$ yields a value $\mathfrak{L}(s, a, V) \leq \mathfrak{B}(s, a, V) = \sum_{\successor \in \States} \mdptransitions(s, a, \successor) \cdot V(\successor)$.
Intuitively, when using $\mathfrak{L}$ instead of $\mathfrak{B}$, the $L_i$ we obtain still are lower bounds on the true value.
If the values returned by $\mathfrak{L}$ additionally \enquote{tighten} over the execution of an algorithm, we converge to the true value.
As further details of this approach are beyond our scope (e.g.\ how to obtain matching upper bounds), we refer to \cite{DBLP:conf/lics/KretinskyMW23,DBLP:conf/aaai/MeggendorferWW25}.

Finally, observe that we assume an \emph{online} setting, meaning that we can gather further samples on demand and actively drive the interaction with the system by choosing which actions to take during sampling.
In contrast, e.g.\ \cite{WatO} assumes to be given a fixed set of samples, with no option of gathering additional ones.
As we see later, this significantly changes which solution approaches are adequate.
For simplicity, we assume to know the support of $\mdptransitions$, i.e.\ whether $\mdptransitions(s, a, s') > 0$, often called \enquote{grey box} \cite{AKW19,BazilleGJS20,agarwal2022pac}.

We distinguish two independent modelling choices. First, \emph{sampling access} specifies how observations are obtained. In the strongest setting, one may sample any state-action pair directly. In the simulation setting, one may only start in the initial state and then sample finite paths by choosing actions online. In the batch setting, one is given a fixed data set of past transitions and cannot actively collect further samples. We work in the simulation setting: samples are obtained along finite paths, but the algorithm may reset to the initial state and may choose actions adaptively.
Second, \emph{knowledge of the transition structure} specifies what is known about the unknown transition function. In the white-box setting, the full transition probabilities are known. In the grey-box setting, the support of every transition distribution is known, but the probabilities are unknown. In the black-box setting, the support is unknown, but one may assume a positive lower bound $p_{\min}$ on all non-zero transition probabilities. Without either support information or such a lower bound, finite samples cannot rule out an unseen transition of arbitrarily small probability; hence PAC guarantees for infinite-horizon reachability are impossible in general.

For the rest of the paper, we focus on the online grey-box setting. Black-box models can be handled by first using part of the confidence budget to learn the support, thereby reducing to the grey-box case; details and minor variants are deferred to \cref{app:assumptions}.
We can formulate our problem as follows.
\begin{linebox}{\centering{}Problem: Online Grey Box MDP-SMC}
	\textbf{Input:} An MDP $\MDP$ with unknown transition probabilities $\mdptransitions$ (but known support) and sampling access, a set of goal states $\Reachset$, a precision requirement $\varepsilon > 0$, and a confidence requirement $\gamma > 0$.

	\noindent
	\textbf{Output:} A value $v$ such that $|v - \val(\initialstate)| \leq \varepsilon$ with confidence at least $\gamma$.
\end{linebox}
\noindent As motivated before, we additionally are interested in keeping the number of required interactions with the system as low as possible.
We deliberately leave \enquote{as low as possible} vague:
It is surprisingly challenging to formally define this notion already for a single distribution.
As observed in e.g.\ \cite[Sec.\ 3.3]{WatO}, how many samples different methods require to achieve a given precision heavily depends on the shape of the distribution (e.g.\ the true probability), and focusing only on worst cases is not necessarily meaningful in practice.

\subsection{Existing Methods and the SOTA MDP-SMC Framework} \label{subsec:existing_methods}
We survey existing methods and categorize them using an informal, unifying framework.
Most approaches work in the grey box setting, i.e.\ for each state-action pair we do not know the precise probabilities but the set of successors.
Relatedly, most approaches are \emph{model-based} (an informal notion, intuitively meaning that the underlying graph is constructed, i.e.\ the space requirement is $\mathcal{O}(|\mdptransitions|)$).
This includes \cite{AKW19,WeiningerGMK21,agarwal2022pac,DBLP:conf/nips/SuilenS0022,BaiDubWie23,DBLP:journals/jair/BadingsRAPPSJ23,DBLP:journals/corr/abs-2310-12248}, which all exhibit a common structure:
They \stepsample{} the MDP some finite number of times, from these samples \stepinfer{} confidence intervals on the transition probabilities, and finally \stepsolve{} the induced IMDP, yielding bounds on the true value.
If the bounds are not tight enough, too little data has been gathered, and they go back to the first step.

However, as mentioned, these methods all suffer from severe drawbacks.
Firstly, most works \enquote{reuse} the entire confidence budget in every \stepinfer{} step, which in general is statistically incorrect, as demonstrated in~\Cref{ex:main:conf_example} and further discussed in \cref{app:sota_bad}. This issue is present in most works on SMC~\cite{WeiningerGMK21,agarwal2022pac,DBLP:conf/nips/SuilenS0022,BaiDubWie23,DBLP:journals/jair/BadingsRAPPSJ23,DBLP:journals/corr/abs-2310-12248}, invalidating their proofs of correctness.  (see \cref{app:sota_bad} for details). 
Only in~\cite{AKW19}, this is (naively) addressed by exponentially distributing confidence over all \stepinfer{} steps, 
ensuring soundness of all \stepinfer{}s by union bound.

\begin{example}
\label{ex:main:conf_example}
Consider an MDP with states $s_1,s_2,s_+,s_-$, initial state $s_1$, target state $s_+$, and sink state $s_-$. There is one action $a$ in $s_1$, with
\begin{align*}
     P(s_1,a,s_2)=p, \qquad P(s_1,a,s_-)=1-p,
\end{align*}
and one action $b$ in $s_2$, with $P(s_2,b,s_+)=1$. Thus the reachability value from $s_1$ is exactly $p$.
Suppose an SMC procedure repeatedly constructs fixed-time confidence intervals and stops once the resulting value interval has width below $\varepsilon$. Even if each individual interval is valid at its pre-specified sample size, the stopping time is data-dependent: the procedure continues precisely until the interval looks sufficiently favourable. Thus the final interval is not covered by the fixed-time guarantee. In particular, for rate-optimal fixed-time intervals containing the empirical mean, the law of the iterated logarithm implies that the empirical estimate will almost surely cross any fixed $\sqrt{1/t}$-width boundary infinitely often. Hence there are data-dependent stopping times at which the returned interval excludes the true value with probability one. This is the core reason why fixed-time confidence intervals cannot simply be reused across repeated \textsc{Infer} calls.
\end{example}
Secondly, existing methods use sub-optimal or incorrect statistical methods (usually the general and overly conservative Hoeffding / Okomato bound).
While \cite{WitS,WatO} tackle this for Markov chain SMC and the \enquote{offline} MDP-SMC setting, none discuss how to adapt to an online setting, where we need to repeatedly draw statistical conclusions.

The case where the topology is unknown, i.e.\ \emph{black box}, intuitively corresponds to not knowing how many sides the coin we are tossing has.
The crux of this setting is to reason about the infinite time behaviour (e.g.\ \enquote{can the goal still be reached?}) based on finite samples.
Notably, without any assumptions, it is impossible to give any guarantees on infinite-horizon behaviour.
Suppose we want the probability of eventually getting heads from repeatedly tossing a coin.
This is 1 if the chance for heads is non-zero, and 0 otherwise.
However, from seeing finitely many tails we cannot conclude that the chance of heads is 0.

Two approaches tackle this case directly.
DQL \cite{DBLP:journals/theoretics/BrazdilCCFKKM0U25} only relies on a lower bound on transition probabilities and upper bound on state and action count.
This however comes at the cost of requiring an astronomical amount of samples already on toy examples \cite[Tbl.~1]{AKW19}.
In contrast, \cite{DBLP:conf/rss/FuT14} requires (among others) knowledge of the \emph{mixing time} of the MDP, which allows to extrapolate from (large) finite samples to their infinite time behaviour.
An a-priori bound for the mixing time is exponential in the size of the MDP and leads to similarly large sample requirements.
As such, the standard approach is to separately learn the topology (using mild assumptions, see e.g.\ \cite[Rem.\ 2]{WatO}) to \enquote{recover} the grey box case (see \cref{app:assumptions} for details).
As this \enquote{reduction} is based on estimating probabilities, it directly profits from our improvements, too.

\section{Improving Online MDP-SMC}

\noindent
In this section, we introduce several approaches to statistical estimations in the online MDP-SMC setting.
Recall that we aim to obtain a safe and converging under-approximation $\mathfrak{L}$ of the Bellman backup in every state-action pair.
We introduce several statistical methods that allow us to obtain such approximations which (i)~are correct with high confidence and (ii)~get more precise the more samples we gather.
Together, this then gives us converging bounds on the true value which are correct with high confidence.

As mentioned, we distribute our confidence over all state-action pairs using the union bound, allowing us to focus on the problem localized to a single state-action pair $s\in \States$ and $a \in\stateactions(s)$.
We model the observations at this state-action pair as an i.i.d.\ process of successor-state samples
$Z=(Z_{\sampletime})_{\sampletime \in \PosNaturals}$ with $Z_{\sampletime} \sim \mdptransitions(s, a)$ for all $\sampletime \in \PosNaturals$,
where each realisation $z_{\sampletime}$ lies in $\States$.
In other words, we focus on samples obtained from a finite-support distribution and fix as underlying probability space infinite sequences drawn i.i.d.\ from this distribution.
Due to space constraints, formal definitions of known notions are provided in \cref{app:defs} and proofs in \cref{app:proof}.
The algorithm for MDP-SMC is outlined in \cref{subsec:cs_mdp_smc}.

\paragraph{Notation.}
Throughout the section, we write $\saprob = \mdptransitions(s, a)$ for the underlying distribution, $\successor \in \States$ for a \emph{successor} state, $\Successors =  \{r \in \States \mid \saprob(r) > 0\}$ for all successors, and $\sampletime \in \PosNaturals$ for a point in time.
Given a sequence $z=(z_{\sampletime})_{\sampletime \in \PosNaturals}$ we write $z_{1:\sampletime}=(z_1, \dots, z_\sampletime)$ for every $\sampletime \in \PosNaturals$.
For every $\successor \in \States$ we denote the transition count and empirical transition probability as $\successes_{\successor}(z_{1:\sampletime}) \coloneqq {\sum}_{i=1}^\sampletime \indicator{z_i = \successor}$ and $\average_{\successor}(z_{1:\sampletime}) \coloneqq \tfrac{1}{\sampletime} \successes_{\successor}(z_{1:\sampletime})$.

\paragraph{Value estimation.}
In the most general form, we aim to find a functor
$\mathcal{V}$
that, given an error threshold and a sequence of observations, maps value functions $V : \States \to \Reals$ over successors to a subset of the real numbers (typically an interval), containing the Bellman backup of $V$.
Formally, we want for every $\error\in(0,1)$ that
\begin{align}
    \label{eq:value_cs}
    \Probability\left[\forall \sampletime \in \PosNaturals. \forall V : \States \to \Reals.\ {\sum}_{\successor \in \States} \saprob(\successor) \cdot V(\successor) \in \mathcal{V}(\error; Z_{1:\sampletime})(V) \right] \geq 1 - \error.
\end{align}
Intuitively, this guarantees that with high confidence all Bellman backups performed using $\mathcal{V}$ contain the true value, allowing to use $\mathfrak{L}(s, a, V) = \inf \mathcal{V}(\error; Z_{1:\sampletime})(V)$.
Observe that the $\forall \sampletime$ is inside the $\Probability$ operator.
This is essential in our online setting, as it permits \emph{adaptively} gathering samples until sufficient precision is achieved.
(We show in \cref{app:proof:confseq} that this requirement is equivalent to requiring correctness at termination.)
Such a value estimation can be obtained by estimating the distribution $\saprob$, as follows.

\paragraph{Distribution estimation.}
Formally, we want a function
$\confset\colon (0,1)\times \States^* \to 2^{\Distributions(\States)}$
that maps a sequence of observations to a \emph{confidence set} of distributions, i.e.\ satisfying
$
    \Probability\left[\forall \sampletime \in \PosNaturals.\ \saprob \in \confset(\error; Z_{1:\sampletime})\right]\ \geq\ 1-\error
$.
From this, we obtain a value estimate by $\mathcal{V}(\error; Z_{1:\sampletime})(V) = \{ \sum_{\successor \in \States} Q(\successor) \cdot V(\successor) \mid Q \in \confset(\error; Z_{1:\sampletime}) \}$.

\paragraph{Probability estimation.}
The common approach to obtain such a set is to estimate each transition probability individually and combine the results.
Formally, for every $\successor$ we seek a function
$\confset_{\successor} \colon (0,1) \times \States^* \to 2^{[0,1]}$
that maps observations to an uncertainty set over the single transition probability, satisfying
\begin{equation}\label{eq:conf_set_trans}
    \Probability\left[\forall \sampletime \in \PosNaturals.\ \saprob(\successor) \in \confset_{\successor}(\error; Z_{1:\sampletime})\right]\ \geq\ 1-\error.
\end{equation}
This gives a distribution estimate by
$\confset(\error; Z_{1:\sampletime}) = \bigtimes_{\successor \in \Successors} \confset_{\successor}(\error/|\Successors|; Z_{1:\sampletime})$ through the union bound. 
Thus, we first focus on estimating a single transition probability.

\subsection{Grid Constructions}
\label{subsec:union_bound}
To estimate individual transition probabilities, one might be tempted to use standard statistical estimates (e.g.\ the Hoeffding inequality) to repeatedly infer bounds.
However, as we outline in \cref{lemma:lil} and \cref{app:sota_bad}, such naive use of fixed-time confidence intervals in an online context (i.e.\ \enquote{reusing} the confidence multiple times) is unsound.
Instead, we suggest to spread the available confidence over an (a-priori fixed) \emph{grid} using the union bound, yielding correctness (\cref{thrm:sound_union}).
Interestingly, the shape of the grid has significant influence on the convergence rate (\cref{lemma:union_convergence}), potentially even leading to divergence (\cref{corr:non_convergence}).

\paragraph{Confidence intervals.}
A (fixed-time) confidence interval $\confinter_{\successor}\colon (0,1)\times \States^* \to 2^{[0,1]}$ guarantees for all $\sampletime \in \PosNaturals$ that $\Probability\left[\saprob(\successor) \in \confinter_{\successor}(\error; Z_{1:\sampletime})\right] \geq 1-\error$.
Crucially, although this guarantee holds for any \emph{a-priori} fixed $\sampletime$, it does \emph{not} imply \cref{eq:conf_set_trans} (observe that the \enquote{for all} is outside $\Probability$).
For example, if we infer at both $\sampletime_1 = 100$ and $\sampletime_2 = 200$, the probability getting at least one estimate wrong may be larger than $\error$.
We show that this occurs for \emph{all} rate-optimal confidence intervals containing the empirical average (see \cref{app:defs:rate_optimal}), including the commonly used Hoeffding~\cite{Hoe63}, (empirical) Bernstein~\cite{MauPon09}, and Clopper-Pearson intervals~\cite{CloPea34} (see \cref{app:defs:intervals}).

\paragraph{Rate-optimal fixed-time intervals.}
Let $I_r : (0,1)\times S^* \to 2^{[0,1]}$ be a fixed-time confidence
interval for the transition probability $P(r)$. We call $I_r$
\emph{empirically anchored} if, for all $t\in\Naturals_+$, all
$\delta\in(0,1)$, and all $z_{1:t}\in S^t$, $ \mu_r(z_{1:t}) \in I_r(\delta;z_{1:t})$.
Its maximal width at time $t$ is
\begin{align*}
    W_{I_r}(\delta,t)
    :=
    \sup_{z_{1:t}\in S^t}
    \bigl(\sup I_r(\delta;z_{1:t})
    -
    \inf I_r(\delta;z_{1:t})\bigr).
\end{align*}
We say that $I_r$ has \emph{rate-optimal maximal width} if there exist
constants $0<c\le C<\infty$, $t_0\in\Naturals_+$, and
$\delta_0\in(0,1)$ such that, for all $t\ge t_0$ and all
$\delta\in(0,\delta_0]$,
\begin{align*}
    c\min\left\{1,\sqrt{\frac{\log(1/\delta)}{t}}\right\}
    \le
    W_{I_r}(\delta,t)
    \le
    C\min\left\{1,\sqrt{\frac{\log(1/\delta)}{t}}\right\}.
\end{align*}
This captures the usual fixed-time
$\sqrt{\log(1/\delta)/t}$ concentration rate of standard Bernoulli
intervals such as Hoeffding, empirical Bernstein, and Clopper--Pearson.

\begin{restatable}{lemma}{lilLemma}
\label{lemma:lil}
For any fixed-time confidence interval
$\confinter_\successor$ 
with rate-optimal maximal width containing the empirical average, there exists an error probability
$\error_0 > 0$ and a transition probability $\saprob(\successor)\in [0,1]$ s.t.\ for every
$\error\in(0,\error_0]$ we have $\Probability\left[
\forall t\in\PosNaturals.\ \saprob(\successor)\in
\confinter_\successor(\error;Z_{1:t})
\right]=0$.
In particular, \cref{eq:conf_set_trans} is violated.
\end{restatable}

\paragraph{Gridding.}
To extend fixed-time guarantees to our setting, one can distribute confidence over time using the union bound.
Essentially, we specify a grid of sample counts at which the confidence interval must be valid and distribute confidence across this grid.
Formally, a grid
$N= (n_i)_{i\in \PosNaturals}$ is an increasing sequence of sample counts $n_i\in\PosNaturals$ and a confidence spending function $h\colon \PosNaturals \to (0,\infty)$ is a function satisfying $\sum_{i=1}^\infty 1/h(i) \leq 1$. Together, we get valid sequential inference.
\begin{restatable}{theorem}{unionGridThm}
    \label{thrm:sound_union}
    Let $\error\in (0,1)$, $\successor \in \States$, $N$ a grid, $h$ a confidence spending function, and $\confinter_{\successor}$ a confidence interval. Then, $\Probability\left[\forall i \in \PosNaturals.\ \saprob(\successor) \in \confinter_{\successor}(\error/h(i); Z_{1:n_i})\right]  \geq 1-\error$.
\end{restatable}
\noindent
More specifically, for any time $\sampletime$ we can use the intersection of all previously inferred sets, i.e.\ $\bigcap_{i, n_i \leq \sampletime} \confinter_{\successor}(\error/h(i); Z_{1:n_i})$.
\Cref{thrm:sound_union} generalizes the approach of \cite{AKW19} for \stepinfer{}, which distributes confidence over times at which \stepinfer{} takes place.

The usefulness of this method depends mainly on the choice of grid and spending function.
In \cref{lemma:union_convergence} we summarise the asymptotic maximal widths for polynomial and exponential gridding, i.e.
$
    N_{\mathrm{poly}}^{b,c} = (\lfloor b\cdot i^c \rfloor)_{i\in \PosNaturals}
$ and $
    N_{\mathrm{exp}}^{b,c} = (\lfloor b\cdot c^i \rfloor)_{i\in \PosNaturals}
$
for $b,c \geq 1$;
and for polynomial and exponential spending, i.e.\ for step $i \in \PosNaturals$ set
$
    h_{\mathrm{poly}}^{a}(i) = \eta_p \cdot i^a
$ and $
    h_{\mathrm{exp}}^a(i) = \eta_e\cdot a^i
$
for $a > 1$,
with spending normalization constants $\eta_p = 1/\sum_{i} i^{-a} = 1/\zeta(a)$ (the Riemann zeta function) and $\eta_e = 1/\sum_i a^{-i}$, s.t.\ $\sum_i 1/h(i)=1$.
\begin{restatable}{theorem}{unionConvergenceLemma}
 \label{lemma:union_convergence}
    Let $\error\in (0,1)$, $\successor \in \States$, $\confinter_\successor$ a confidence interval with rate-optimal maximal width, and $a, b,c\in \Reals$ with $a>1$ and $b,c\geq 1$.
    Applying polynomial and exponential grids under polynomial and exponential spending impacts the rate-optimal maximal widths (and its constants) as follows:
    \begin{center}
        \begin{tabular}{ccc}
& $h_{\mathrm{poly}}^a$ & $h_{\mathrm{exp}}^a$ \\
\midrule
$N_{\mathrm{poly}}^{b,1}$ &
$ \sqrt{(a\log (n_i/b) +\log(\eta_p/\error)) / n_i}$ &
$ \sqrt{((n_i/b)\log a +\log(\eta_e/\error)) / n_i}$ \\
$N_{\mathrm{poly}}^{b,c}$ &
$ \sqrt{((a/c)\log (n_i/b) +\log(\eta_p/\error)) / n_i}$ &
$ \sqrt{(\sqrt[c]{n_i/b}\,\log a +\log(\eta_e/\error)) / n_i}$ \\
$N_{\mathrm{exp}}^{b,c}$ &
$ \sqrt{(a\log \log_c(n_i/b) +\log(\eta_p/\error)) / n_i}$ &
$  \sqrt{(\frac{\log(n_i/b)}{\log(c)}\log a +\log(\eta_e/\error)) / n_i}$.
\end{tabular}
    \end{center}
\end{restatable}
\noindent
We note two disadvantages of the exponential spending (as e.g.\ used in~\cite{AKW19}): Firstly, the interval width is (asymptotically) at least as large as for polynomial spending (\cref{lemma:poly_dominance}).
Secondly, while the intervals converge for every grid size under polynomial spending, a constant grid size for exponential spending leads to non-convergence (\cref{corr:non_convergence}).%
\begin{restatable}{lemma}{polyDominanceLemma}
    \label{lemma:poly_dominance}
    Let $\error\in (0,1)$, $\successor \in \States$, $\confinter_{\successor}$ a confidence interval with rate-optimal maximal width, and $N$ a grid. Then, for every $a,b\in \Reals$ such that $a,b>1$, we get for large $i\in \PosNaturals$ that $  W_{\confinter_{\successor}}(\error/h_{\mathrm{poly}}^a(i),n_i)\leq W_{\confinter_{\successor}}(\error/h_{\mathrm{exp}}^{b}(i),n_i)$.
\end{restatable}
\begin{restatable}{corollary}{nonConvergence}
    \label{corr:non_convergence}
    For every error probability, every confidence interval with rate-optimal maximal width, every grid with constant spacing, and every exponential spending function, the resulting maximal widths fail to converge to $0$.
\end{restatable}

\noindent In general, \cref{lemma:union_convergence} shows that slower spending and sparser grids yield a better asymptotic convergence rate.
However, sparse grids in turn increase the gaps between time points where estimates can be refined.
This trade-off between asymptotic convergence rate and large gaps arises because the naive union bound treats each grid point as an independent experiment.
Fortunately, there are methods leveraging the connection between adjacent points to provide guarantees for \emph{every} time point.

\subsection{Confidence Sequences}
\label{subsec:conf_sequences}
A confidence sequence~\cite{darling1967confidence} is a stream of probably correct estimates from a stream of samples, as defined in \cref{eq:conf_set_trans}.
Notably, a gap-free grid ($n_i=i$) with polynomial spending already gives a confidence sequence, however with a (sub-optimal) convergence rate of $\sqrt{(\log t + \log(1/\error))/t}$ (\cref{lemma:union_convergence}), whereas $\sqrt{(\log \log t + \log(1/\error))/t}$ is optimal~\cite{howard2021time}.
Interestingly, this in turn matches the rate obtained by exponentially spaced grids with polynomial spending, while providing validity at \emph{all} times.
Many confidence sequences with optimal rates are fundamentally connected to martingale theory and Ville's inequality~\cite{howard2020time} (with some exceptions~\cite{jamieson2014lil}).
As with confidence intervals, there are worst-case, variance-aware, and parametric confidence sequences, mirroring the fixed-time Hoeffding, (empirical) Bernstein, and Clopper-Pearson constructions.
Similar in spirit to \cref{subsec:union_bound}, \emph{stitching} \cite{howard2021time} first applies a union bound on an exponentially spaced grid and then uses Ville's inequality to obtain validity at intermediate times.
In~\cite{waudby2024estimating}, the authors construct \emph{betting} confidence sequences from test supermartingales by mirroring the classical Chernoff method for constructing confidence bounds.
\cref{app:defs} provides formal definitions of the known \emph{Stitched Hoeffding} $\confseq_\successor^{SH}$, \emph{Stitched Bernstein} $\confseq_\successor^{SB}$, \emph{Betting Hoeffding} $\confseq_\successor^{BH}$, and \emph{Betting Bernstein} $\confseq_\successor^{BB}$ sequences.

\begin{restatable}{theorem}{csFamilyThm}
    \label{thrm:sound_var_cs}
    $\confseq_\successor^{SH}$, $\confseq_\successor^{SB}$, $\confseq_\successor^{BH}$, and $\confseq_\successor^{BB}$ are confidence sequences (\cref{eq:conf_set_trans}).
\end{restatable}
\noindent Therefore, we can directly use these to tackle the probability estimation problem.
We note that the betting confidence sequences are built from test \emph{super}martingales, introducing additional slack compared to test martingales.
Test martingale-based betting confidence sequences, such as those  of~\cite{waudby2024estimating}, typically lack closed forms and are costly to update, scaling linearly with the sample history.
Fortunately, our setting admits efficiently computable likelihood-ratio-based test martingales.

\subsection{Test Confidence Sequences}\label{subsec:test-conf-sequence}
A test-based confidence sequence is obtained by inverting a family of sequential hypothesis tests.
We use the general approach outlined \cite{ramdas2025hypothesis,waudby2024estimating} to construct two test-based confidence sequences: One for a single transition probability $\mdptransitions(s, a, \successor)$, as natural analogue to the Clopper-Pearson interval, and one for the entire distribution $\mdptransitions(s, a)$, avoiding the additional slack introduced by the union bound. 
Although both test-based confidence sequences lack a closed form, we show that they can be utilised efficiently during value iteration by solving a one dimensional optimisation problem via bisection.

\paragraph{Single-transition test confidence sequence.}
For an intuition, fix a sequential hypothesis test $\psi_\successor \colon [0,1]\times \States^* \to \{0,1\}$ that, for a given parameter $q \in [0,1]$, tests the null $\mathcal{H}_0\colon \saprob(\successor)=q$ against the alternative $\mathcal{H}_1\colon \saprob(\successor)\neq q$ with a false positive rate of at most $\error\in (0,1)$.
In other words, conditioned on $\saprob(\successor)=q$ we have $\Probability[ \exists \sampletime\in \PosNaturals.\ \psi_\successor(q; Z_{1:\sampletime})=1] \leq \error$.
Hence, if $\psi_\successor(q;z_{1:\sampletime})$ is $1$ at some point in time $\sampletime$, the decision to remove $q$ out of the set of possible parameter for $\saprob(\successor)$ is the correct decision with probability greater than $\gamma=1-\delta$.
This yields a confidence sequence by considering exactly those parameters that have not yet been rejected, i.e.\ $\confset_r(\gamma; Z_{1:\sampletime}) = \{q \in [0,1]\mid \psi_\successor(q;Z_{1:\sampletime})=0\}$.

A natural choice of $\psi$ for a single transition is based on the likelihood ratio~\cite{wald1992sequential,waudby2024estimating}.
Specifically, for the observations $z_{1:\sampletime}\in \States^\sampletime$ let $x_i \coloneqq \indicator{z_i=\successor}$ indicate a transition to $\successor\in \Successors$, then we define for $q\in (0,1)$ 
\begin{align} 
\label{eq:MLE}
    M_\successor^{\mathrm{MLE}}(q; z_{1:\sampletime})
     \coloneqq
     {\prod}_{i=1}^\sampletime \frac{\average_\successor(z_{1:i-1})^{x_i} (1-\average_\successor(z_{1:i-1}))^{1-x_i}}{q^{x_i}(1-q)^{1-x_i}},
\end{align} 
The function $ M_\successor^{\mathrm{MLE}}$ computes the product of likelihood ratios, where the alternative parameter $p$ at time $i$ is replaced by the \emph{predictable} empirical average up to time $i-1$, i.e.\ the maximum likelihood estimate before observing the new data at time $i$.
Using the predictable average, as suggested by~\cite{waudby2024estimating,ramdas2025hypothesis}, is the crucial difference to the classical Wald Sequential Probability Ratio Test~\cite{wald1992sequential}, as it ensures 
that the process $M_\successor^{\mathrm{MLE}}(\saprob(\successor); \cdot)$ is a test-martingale, i.e.\ a non-negative martingale with expectation $1$.
By Ville's inequality we know that this process never exceeds $1/\error$ with a probability greater than $\gamma$, i.e.\ 
$
\Probability[\exists \sampletime \in \PosNaturals .\ M_\successor^{\mathrm{MLE}}(\saprob(\successor); Z_{1:\sampletime})  \geq \tfrac{1}{\error}] \leq \error
$.
Conversely, if the threshold is exceeded, the assumption that $\saprob(\successor)=q$ is rejected with probability $\gamma$.
Hence, $M_\successor^{\mathrm{MLE}}(q; z_{1:\sampletime})>1/\error$ is an appropriate hypothesis test. 

Interestingly, this method is not restricted to the predictable empirical average.
In fact, \emph{any} (history-dependent) prediction of $\saprob(\successor)$ is sufficient:
Let $\nu_\successor \colon \States^* \to \Distributions([0,1])$ map observations to a parameter distribution, define for $q\in (0,1)$
\begin{equation}
    \label{eq:LRmart}
    M_\successor^{\nu_\successor}(q; z_{1:\sampletime})
    \coloneqq
    {\prod}_{i=1}^\sampletime \left(\int_{0}^1 \frac{u^{x_i} (1-u)^{1-x_i}}{q^{x_i}(1-q)^{1-x_i}} \, \nu_\successor(z_{1:i-1})(du) \right).
\end{equation}
and $M_\successor^{\nu_\successor}(0; z_{1:\sampletime})=1$ if $\successes_{\successor}(z_{1:\sampletime})=0$, $M_\successor^{\nu_\successor}=1$ if $\successes_{\successor}(z_{1:\sampletime})=\sampletime$, and $M_\successor^{\nu_\successor}(q; z_{1:\sampletime})=+\infty$ for $q\in \{0,1\}$ otherwise. 
Intuitively, $M_\successor^{\nu_\successor}$ is the product of likelihood ratios, averaged over alternatives $u$ with respect to a history-dependent distribution representing a best-effort prediction of $\saprob(\successor)$.
An easily computable example of $\nu_\successor$ is the Beta distribution, i.e.\ for $a_0,b_0\in \Reals_{>0}$, $a_\sampletime  \coloneqq a_0 + \successes_\successor(z_{1:\sampletime})$, $b_\sampletime \coloneqq b_0 + (\sampletime-\successes_\successor(z_{1:\sampletime}))$
\begin{equation*}
     M_\successor^{\mathrm{Beta}}(q;z_{1:\sampletime})
     \coloneqq
     {\prod}_{i=1}^\sampletime
     \frac{\mathrm{Beta}(a_{i-1}+ x_i, b_{i-1}+(1-x_i))}{\mathrm{Beta}(a_{i-1},b_{i-1})}\cdot \frac{1}{q^{x_i}(1-q)^{1-x_i}}.
\end{equation*}

\begin{restatable}{theorem}{lrSingleThm}
\label{thrm:sound_single_lr}
Let $\error\in (0,1)$, $\successor \in \Successors$, and $M_\successor^{\nu_\successor}$ as in \cref{eq:LRmart}. Then $\confset_\successor^{M_\successor^{\nu_\successor}}$ is a valid confidence sequence (satisfying \cref{eq:conf_set_trans}) and if $\nu_\successor(Z_{1:\sampletime})$ converges to $\saprob(\successor)$ a.s.\ then $\confset_\successor^{M_\successor^{\nu_\successor}}$ converges to $\saprob(\successor)$.
\end{restatable}

\paragraph{Computability.}
Since these test-based confidence sets are not represented as a closed form expression, computability is not obvious.
To this end, we exploit the convexity of the log-likelihood ratio in the parameter $q$.
\begin{restatable}{lemma}{alphaConvexLemma}
\label{lemma:convex}
Let $\error\in (0,1)$, $\successor \in \Successors $, and $M_\successor^{\nu_\successor}$ as in \cref{eq:LRmart}.
Then $\log M_\successor^{\nu_\successor}(q, z_{1:\sampletime})$ is convex in $q \in (0, 1)$ with minimum at $\average_\successor(z_{1:\sampletime})$ when $\average_\successor(z_{1:\sampletime}) \in (0,1)$. 
\end{restatable}
\noindent
With this, the confidence sequence $\{ q\in [0,1] \mid M_\successor^{\nu_\successor}(q; z_{1:\sampletime}) \leq 1/\error\}$ is an interval and 
its endpoints can be found by bisection.

\paragraph{Probability-distribution test confidence sequence.}
The likelihood-ratio approach can be extended to sequentially valid confidence sets for the full successor distribution $\saprob$, by replacing the multinomial likelihood ratio instead of binomial one.
This method avoids the union bound required to reconstruct the successor distribution and therefore can be statistically more efficient, especially for states with large successor count.

In other words, instead of using a hypothesis test for a single probability, we consider a hypothesis test for the entire distribution.
As above, we provide a general definition using a (history-dependent) prediction for it.
Let $\nu \colon \States^* \to \Distributions(\Distributions(\Successors))$ map state observations to a distribution over successor distributions.
For $\sampletime \in \PosNaturals$, $Q$ in the interior of $\Distributions(\Successors)$, $U\in\Distributions(\Successors)$ and $z_{1:\sampletime}\in \States^\sampletime$, define
\begin{align}
    \label{eq:LRset}
    M^{\nu}(Q;z_{1:\sampletime})
    \coloneqq
    {\prod}_{i=1}^\sampletime \int_{\Distributions(\Successors)} \frac{U(z_i)}{Q(z_i)}\, \nu(z_{1:i-1})(dU).
\end{align}
and for $Q$ on the boundary of $\Distributions(\Successors)$ analogously to~\Cref{eq:LRmart}. 
For example, in $M^{\mathrm{MLE}}(Q;z_{1:\sampletime})$, $\nu$ maps to the empirical distribution
$
    \hat{P}(z_{1:\sampletime}) \coloneqq (\average_\successor(z_{1:\sampletime}))_{\successor \in \Successors}
$.
\begin{restatable}{theorem}{lrMultiThm}
\label{thrm:sound_multi_lr}
    Let $\error\in (0,1)$ and $M^{\nu}$ as in \cref{eq:LRset}.
    Then $ \confset^{M^{\nu}}$ is a valid confidence sequence, i.e.\  $
         \Probability\left[\forall \sampletime \in \PosNaturals.\ \saprob \in \confset^{M^{\nu}}(\error; Z_{1:\sampletime}) \right] \geq 1-\error
    $
    and if $\nu(Z_{1:\sampletime})$ converges to the successor distribution $\saprob$ a.s. then  $\confset^{M^{\nu}}$ converges to $\saprob$.
\end{restatable}

\paragraph{Computability.}
Before, we could explicitly represent the one-dimensional uncertainty set for $\saprob(\successor)$.
This is no longer clear for the $|R|$-dimensional confidence set $\confset^{M^\nu}(\error; z_{1:\sampletime})=\{Q\in\Distributions(\Successors) \mid M^\nu(Q;z_{1:\sampletime}) \leq 1/\error\}$.
However, recall that ultimately we only need to compute bounds on the Bellman backup, i.e.\ for a value function $V \colon \States\to \Reals$ compute $\inf_{Q \in \confset^{M^\nu}(\error; z_{1:\sampletime})} \sum_{\successor \in \Successors} Q(\successor) \cdot V(\successor)$.
To solve this optimisation problem, we show that $\confset^{M^\nu}(\error; z_{1:\sampletime})$ is the intersection of the simplex with a convex sublevel set of $\alpha(Q;z_{1:\sampletime})$.
Therefore, the minimiser is the solution of a convex optimisation problem. Moreover, we show that the KKT conditions~\cite[p.~243]{boyd2004convex} apply and the optimisation problem admits a one-dimensional reduction.

\begin{restatable}{theorem}{kktProjectionThm}
\label{thrm:kkt_projection}
Let $\sampletime \in\PosNaturals$, $\error\in(0,1)$, $z_{1:\sampletime}\in\States^\sampletime$, and $V\colon\Successors\to\Reals$.
Assume $\successes_\successor(z_{1:\sampletime})>0$ for all $\successor \in \Successors$ (otherwise restrict to $\{\successor \mid \successes_\successor(z_{1:\sampletime}) > 0\}$ and 
assign all remaining mass to an extremiser of $V$ among $\{\successor \mid \successes_\successor(z_{1:\sampletime}) = 0\}$).
If $V$ is not constant and $\confset^{M^\nu}(\error; z_{1:\sampletime})\neq \emptyset$, then 
\begin{align*}
 Q^\star_\eta \in  \argmin_{Q \in \confset^{M^\nu}(\error; z_{1:\sampletime}) } \sum_{\successor\in \Successors} Q(\successor) \cdot V(\successor) =
    \begin{cases}
        \hat{P}(z_{1:\sampletime}) \quad
        \text{if  $M^\nu(\hat{P}(z_{1:\sampletime});z_{1:\sampletime})= 1/\error$ }\\
        \frac{\successes_\successor(z_{1:\sampletime})}{\eta+V(\successor)}/ \sum_{u\in\Successors}\frac{\successes_u(z_{1:\sampletime})}{\eta+V(u)} \quad \text{otherwise}
    \end{cases},
\end{align*}
where $\eta>-\min_{\successor \in \Successors} V(\successor)$ is the unique solution of $ M^\nu(Q^\star_\eta;z_{1:\sampletime})=1/\error $ and $M^\nu(Q^\star_\eta;z_{1:\sampletime})$ is strictly decreasing in $\eta$.

\end{restatable}
\noindent
Consequently,
$Q^\star_\eta$ can be found efficiently using standard bisection.

\subsection{Confidence Sequence MDP-SMC} \label{subsec:cs_mdp_smc}
We now summarize our new approach towards online MDP-SMC.
Recall that the \stepsample{}-\stepinfer{}-\stepsolve{} pipeline first distributes confidence over inference steps, and then in each inference step distributes the available confidence over state-action pairs to apply point-wise estimates (such as the Hoeffding bound), i.e.\ spreading confidence first over time, then over (state) space.
Instead, we propose to invert this order by first distributing confidence over all state-action pairs and then constructing confidence sequences for each of them, i.e.\ spreading confidence first over space and then time.
Instead of breaking the process into discrete steps, this allows us to interleave sampling and solving, akin to partial exploration approaches \cite{DBLP:journals/theoretics/BrazdilCCFKKM0U25,DBLP:conf/atva/Meggendorfer22,DBLP:conf/cav/MeggendorferW24}, as we obtain reliable and improving estimates after every sample.
In particular, we can refine value estimates after every step.

Formally, we distribute the entire error budget $\error_{\mathrm{total}}$ over all state-action pairs, obtaining a separate error budget for each individual pair.
For example, using the union bound, we get $\error = \frac{\error_{\mathrm{total}}}{|\{(s, a) \mid s \in \States, a \in \stateactions(s)\}|}$.
Even though we sample paths through the MDP, by the Markov property we can apply our methods independently in each pair using the per-pair budget and the successor state observations gathered there.
Concretely, whenever we visit a state $s$ and choose action $a$ in it, we sample a successor and \enquote{advance} the inference process for this specific state-action pair, and use the information in the next Bellman backup of the robust value iteration.
See also \cite[Chp.~5.4.3]{wienhoft2025statistical} for further details.
\section{Implementation and Experimental Evaluation}
\label{sec:experiments}
We present our implementation and evaluation, focusing on two main questions:
\begin{description}
    \item[RQ1] How do the presented inference methods perform on individual distributions, and how does the shape of distributions influence their performance?
    \item[RQ2] How do our methods perform in online MDP-SMC tasks, i.e.\ how many samples do we require until we obtain the required precision?
\end{description}
We expect correlation between the answers to these two questions, i.e.\ inference methods which perform well on individual distributions should also do so in MDP-SMC, but not necessarily a 1:1 correspondence.
Firstly, some inference methods are more suitable for certain shapes of distributions (empirically observed in \textbf{RQ1} below), and so similarly they are better for MDPs where a particular distribution-shape is prevalent.
Secondly, the random nature of SMC, further emphasized by our guided sampling approach (explained below), introduces significant variance to the overall process.
We deliberately do not emphasize runtime, as gathering samples in real-world scenarios typically is more expensive than their analysis.

\subsection{Implementation Details}
We implemented all presented methods in a Java tool.
For handling MDP, we build on the library of the \emph{partial exploration tool} (PET) \cite{DBLP:conf/cav/MeggendorferW24} and support models given in PRISM language \cite{kwiatkowska2011prism} with reachability properties.

\subsubsection*{Inference methods.}

We list all implemented inference methods with a short name for referring to them.
We consider the Hoeffding (\texttt{Hoeff}), Bernstein (\texttt{Bern}), and Clopper-Pearson (\texttt{CP}) inequalities.
From \cref{subsec:union_bound}, we get \enquote{sequentialized} versions using exponential ($h_{\textrm{exp}}^2$) and squared spending ($h_{\textrm{poly}}^2$) of confidence (indicated by the suffix \texttt{-Exp} or \texttt{-Sq}, e.g.\ \texttt{CP-Sq}).
To ensure convergence for \texttt{-Exp} (see \cref{corr:non_convergence}), we use an exponential grid size (i.e.\ $n_i - n_{i-1} = 2^i$); for squared we choose a gap-less grid ($n_i = i$) to compare to confidence sequences.
From \cref{subsec:conf_sequences}, we use the stitched and betting variants of Hoeffding- and Bernstein-based sequences (\texttt{Hoeff-Stitch}, \texttt{Hoeff-Bet}, \texttt{Bern-Stitch}, and \texttt{Bern-Bet}).
From \cref{subsec:test-conf-sequence}, we use the maximum likelihood (\texttt{MLETest}) and Beta hypothesis test methods (\texttt{BetaTest}).
For \texttt{BetaTest}, we deliberately used uniform priors ($a_0 = b_0 = 1$) to avoid fitting to our dataset.
Finally, we have \texttt{ValueTest} using \cref{thrm:kkt_projection}.

For validation, we compared results with an independent Python implementation on randomly generated inference tasks, and used parametrized hypothesis tests on the correctness of each inference method (e.g.\ statistically testing the hypothesis that the inferred probabilities are correct with probability at least $\gamma$).

\subsubsection*{MDP-SMC approaches.}
We implement our proposed confidence sequence approach as described in \cref{subsec:cs_mdp_smc} as well as the state of the art approach.
For the former, for every confidence sequence \texttt{X}, we get an instance \texttt{CS-X} of our approach.
For the latter, recall that it is parametrized by (i)~how far apart the \stepinfer{} steps are (i.e.\ grid size), (ii)~how the confidence is spent, and (iii)~which point-wise estimate is used.
We refer to our implementation of this approach by \texttt{SOTA-X-Y}, where \texttt{X} describes grid and spending and \texttt{Y} the applied (point-wise) inference method.
For \texttt{X}, we use exponential (\texttt{Exp}), i.e.\ $\delta \cdot h_{\mathrm{exp}}^2(i) = \frac{\delta}{2^i}$ confidence and $100 \cdot 2^{i}$ sampled paths in iteration $i$, and squared (\texttt{Sq}), i.e.\ $\delta \cdot h_{\mathrm{poly}}^2(i) = \delta \cdot \frac{6}{i^2 \pi^2}$ confidence and $100 \cdot i^2$ sampled paths in iteration $i$.

Both variants additionally rely on a \emph{sampling strategy}, i.e.\ which actions to choose while sampling a path.
We employ guided sampling as in \cite{DBLP:journals/theoretics/BrazdilCCFKKM0U25}:
We choose actions which currently achieve the highest value in the corresponding state
until a goal or sink state is reached.
For updating value estimates, we employ robust value iteration \cite{DBLP:conf/aaai/MeggendorferWW25}, continuing until a value estimate is $\varepsilon$-precise.
We pre-process MDPs by collapsing end components, as suggested in~\cite{WatO}.

Finally, we consider \texttt{CAV19}~\cite{AKW19}, the only other sound approach we know of.
From a statistical perspective, it is similar to \texttt{SOTA-Exp-Hoeff}, however with technical differences, such as the grid spacing ($10^5 \cdot 2^i$ instead of $100 \cdot 2^i$) and sampling strategy.
As the original code contained a bug (not properly distributing confidence), we reimplemented the method, mimicking the implementation as closely as possible. 
The changes were bug fixes and adding the improvements of~\cite{WatO} to ensure a fair comparison of the statistical methods.

\subsection{RQ1: Performance of Individual Inference}
In this section, we investigate our first research question, i.e.\ how the different sequential inference methods perform on various distributions.

\newcommand{\plotmarksize}{1.5pt}
\newcommand{\markincaption}[2]{{\begin{tikzpicture}[baseline=-3pt] \protect\draw[#1,ultra thick,mark size=3pt] plot[mark=#2] (0,0);\end{tikzpicture}}}

\pgfplotsset{
    sampleplot/.style={
		table/col sep=comma,
		x label style={anchor=north,inner sep=0pt},
		y label style={anchor=south,inner sep=0pt},
        enlarge x limits=false,
        enlarge y limits=false,
		axis x line*=bottom,
		axis y line*=left,
        cycle list/Paired,
        no marks,
        thick
    }
}

\newcommand{\singleplot}[4][]{
    \addplot+[#1] table [x=steps, y=#3] {data/sample_data_#2.csv};
    \addlegendentry{\texttt{#4}}
}
\newcommand{\skipplot}{
\addplot+ coordinates {(0,0)};
}

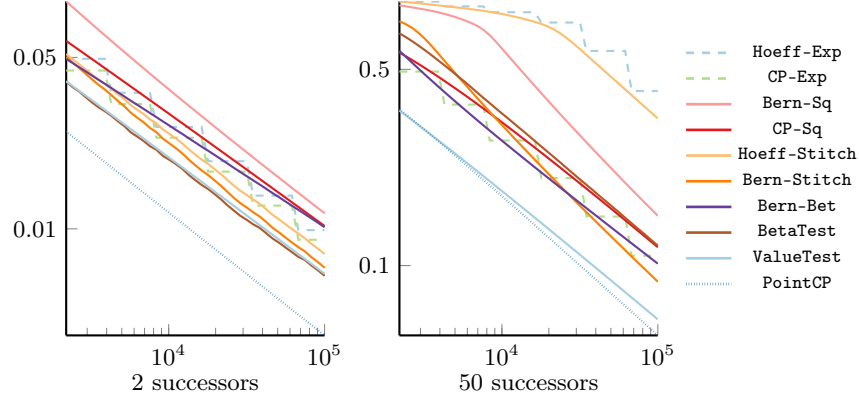
\begin{figure}[t]
    \centering
    \begin{tikzpicture}
    \begin{axis}[
            width=5cm,
    		height=6cm,
            every axis legend/.code={\let\addlegendentry\relax},
            sampleplot,
            ytick={0.05,0.01},
            yticklabels={0.05,0.01},
		      xlabel=2 successors,
            xmode=log,ymode=log,
            EvalList
    	]
        \singleplot{Uniform_2_0.01}{Hoeffding Union-Exp}{Hoeff-Exp}
        \skipplot
        \singleplot{Uniform_2_0.01}{CP Union-Exp}{CP-Exp}
        \skipplot
        \singleplot{Uniform_2_0.01}{Bernstein Union-Sq}{Bern-Sq}
        \singleplot{Uniform_2_0.01}{CP Union-Sq}{CP-Sq}
        \singleplot{Uniform_2_0.01}{Hoeffding Stitched}{Hoeff-Stitch}
        \singleplot{Uniform_2_0.01}{Bernstein Stitched}{Bern-Stitch}
        \skipplot
        \singleplot{Uniform_2_0.01}{Bernstein Betting}{Bern-Bet}
        \skipplot
        \singleplot{Uniform_2_0.01}{Beta Test}{BetaTest}
        \singleplot{Uniform_2_0.01}{Value Test}{ValueTest}
        \singleplot[dash pattern=on 0.2pt off .8pt]{Uniform_2_0.01}{Point CP}{PointCP}
    \end{axis}
    \end{tikzpicture}%
    \begin{tikzpicture}
    \begin{axis}[
            width=5cm,
    		height=6cm,
            sampleplot,
            xmode=log,ymode=log,
		      xlabel=50 successors,
            legend style={
                font=\scriptsize,
                draw=none,
                row sep=0pt,
                at={(axis description cs:1.1,0.5)},
                anchor=west,
                inner sep=0pt
            },
            ytick={1.0,0.5,0.1},
            yticklabels={1.0,0.5,0.1},
            legend image post style={line width=1pt},
            EvalList
    	]
        \singleplot{Uniform_50_0.01}{Hoeffding Union-Exp}{Hoeff-Exp}
        \skipplot
        \singleplot{Uniform_50_0.01}{CP Union-Exp}{CP-Exp}
        \skipplot
        \singleplot{Uniform_50_0.01}{Bernstein Union-Sq}{Bern-Sq}
        \singleplot{Uniform_50_0.01}{CP Union-Sq}{CP-Sq}
        \singleplot{Uniform_50_0.01}{Hoeffding Stitched}{Hoeff-Stitch}
        \singleplot{Uniform_50_0.01}{Bernstein Stitched}{Bern-Stitch}
        \skipplot
        \singleplot{Uniform_50_0.01}{Bernstein Betting}{Bern-Bet}
        \skipplot
        \singleplot{Uniform_50_0.01}{Beta Test}{BetaTest}
        \singleplot{Uniform_50_0.01}{Value Test}{ValueTest}
        \singleplot[dash pattern=on 0.2pt off .8pt]{Uniform_50_0.01}{Point CP}{PointCP}
    \end{axis}
    \end{tikzpicture}
    \caption{
        Comparison of our presented methods on (nearly) uniform distributions with 2 and 50 successors, respectively.
        We show how the inference error (averaged over 40 random value vectors, y-axis) decreases with increasing number of samples (x-axis).
        Observe that both axes are logarithmic.
        We only consider a selection of methods for readability, further data can be found in \cref{app:data:individual}
    }
    \label{fig:results_individual}
\end{figure}

\paragraph{Setup.}
We consider several classes of distributions, parametrized by the number of successors and their \enquote{shape}.
\cref{fig:results_individual} shows results for (nearly) uniform distributions with 2 and 50 successors; further data can be found in \cref{app:data:individual}.
For each class, we generate 20 random distributions and 40 random value vectors.
Then, for each distribution, we generate \num{100000} samples, feed these to the inference methods with an overall confidence of $\gamma = 0.9$, and compute the difference between the smallest and largest value estimate each inference method can provide for each of the given vectors at several time steps.
For each method, we then report the average of this error.
Note that we cannot give an error estimate on probabilities directly, since, for example, the \texttt{ValueTest} approach does not estimate probabilities.
To establish a lower bound, we also compare with the pointwise Clopper-Pearson estimate (\texttt{PointCP}), i.e.\ what we could infer if we used all given confidence at a single point in time instead of spreading it over infinite time.

\paragraph{Results.}
\cref{fig:results_individual} summarizes our results, omitting some inference methods for readability.
We see that the number of successors (and shape, as visible in \cref{app:data:individual}) of the distribution has a profound impact on the rate of convergence.
Moreover, confidence sequences, in particular \texttt{Bern-Stitch} and \texttt{ValueTest}, perform favourably.
Notably, our novel \texttt{ValueTest} is always among the best and sometimes even more precise than \texttt{PointCP}, as it does not need to infer individual probabilities but directly infers a single value.
In other words, we lose surprisingly little precision for gaining guarantees across arbitrarily many samples instead of just a single point in time.

\subsection{RQ2: Performance of Online MDP-SMC}
In this section, we investigate our second research question, i.e.\ the impact of using our inference methods in the context of online MDP-SMC.

\begin{table}[t]
\newcommand{\rot}[1]{\rotatebox[origin=r]{-90}{\texttt{\small #1}}}
    \centering
    \caption{
        Comparison scores of all our methods (lower is better), sorted from best to worst.
        The scores describe how much more samples (relatively) a method needs on average compared to the best method.
        For example, a score of 10 means a method required 10 times more samples than the best on average.
    }
    \label{tab:mdp_smc_data} %
    \begin{tabular}{rcccccccccccccccccccc}
 Method   &   {\rot{CS-ValueTest}} &   {\rot{CS-BetaTest}} &   {\rot{CS-MLETest}} &   {\rot{CS-CP-Exp}} &   {\rot{CS-Bern-Bet}} &   {\rot{CS-CP-Sq}} &   {\rot{CS-Bern-Stitch}} &   {\rot{CS-Hoeff-Stitch}} &   {\rot{CS-Hoeff-Exp}} &   {\rot{CS-Hoeff-Sq}} &   {\rot{CS-Bern-Exp}} &   {\rot{CS-Bern-Sq}} &   {\rot{SOTA-Sq-CP}} &   {\rot{SOTA-Sq-Hoeff}} &   {\rot{SOTA-Sq-Bern}} &   {\rot{CS-Hoeff-Bet}} &   {\rot{SOTA-Exp-CP}} &   {\rot{SOTA-Exp-Hoeff}} &   {\rot{SOTA-Exp-Bern}} &   {\rot{CAV19}} \\
\midrule
 Score    &                 1.3 &                1.5 &               1.5 &              1.6 &                1.8 &               2 &                   2.1 &                    2.2 &                 2.7 &                3.3 &                4.1 &                 5.0 &                    5.9 &                       6.8 &                      8.6 &                  11 &                      19 &                         22 &                        23 &              50 \\
    \end{tabular} %
\end{table}

\paragraph{Setup.}
We evaluate all methods on several MDP benchmarks, using our implementation of \cite{AKW19} as baseline.
We consider MDPs from the PRISM benchmark suite \cite{prism-benchmarksuite} and the (revised) practitioner's guide's community set \cite{DBLP:conf/tacas/HartmannsJQW23,hartmanns2025revised} with a non-trivial reachability query (including all models in \cite{WatO} and more).
For models with multiple parametrizations or multiple non-trivial queries, we (arbitrarily) selected one of them to avoid over-representing one particular model's structure.
All in all, we obtain 320 model-method combinations.
As metric, we focus on the total number of samples.
To smooth randomness, we repeat each experiment ten times and report the average.
We focus on moderately sized models to allow for a reasonably quick reproduction, and set a timeout of 120s per run.

\paragraph{Results.}
We summarize our findings in \cref{tab:mdp_smc_data}.
For brevity, we present an average score of each method, obtained as follows.
For each model, we compute the relative difference between the method and the best one and then take the (geometric) average of this ratio over all models.
For example, a score of 4 intuitively means that, on average, this method required four times more samples than the respective best method.
Timeouts are counted as twice the worst method's performance.
We treat all sample counts below 100 as equally good, to reduce the impact of minor fluctuations.
Individual results are provided in \cref{app:data:mdp-smc}.

Our new methods clearly outperform both (our improved implementation of) the previous approach of \cite{AKW19} and all other state-of-the-art variants.
We conjecture that this is due to confidence sequences being able to consistently give meaningful information to the interwoven solving and guidance, allowing the approach to focus on regions of interest more effectively.
Of our methods, the hypothesis test methods are the clear winners.
Notably, if we exclude the one model where \texttt{ValueTest} times out, it achieves a near-perfect score of 1.1.

Moreover, the dedicated confidence sequences (mostly) outperform the \enquote{sequentialized} approaches.
The only exception are the sequentialized Clopper-Pearson methods, which beat some of the confidence sequences.
We conjecture that this is due to Clopper-Pearson being specifically tailored towards binomial distributions.

For runtime, we found that as soon as gathering samples becomes non-trivial, their cost vastly dominates any computational effort required for the statistical estimates.
Already in our unoptimized implementation, each method is able to handle thousands of updates per second.

\paragraph{Overall Verdict.}
In summary, in most cases our novel \texttt{ValueTest} is the best choice, as typically gathering samples is much more expensive than a slight increase in runtime.
Given appropriate domain knowledge, methods using priors (such as \texttt{BetaTest}) may perform even better.
Finally, in cases where gathering samples is extremely cheap and overall runtime is of importance, easily computable approaches such as \texttt{Bern-Bet} may be the best choice.
\section{Conclusion}

We have presented confidence sequences for online MDP-SMC, improving upon the previous, mostly flawed and/or inefficient state of the art.
These sequences obtain statistical estimates that improve with every arriving sample while maintaining global guarantees.
Thus, they are an ideal match for online inference tasks, where gathering samples and extracting knowledge from them is interwoven.
The efficiency of these methods is underlined by our experimental evaluation.

For future work, we are interested in further improving on the performance and numerical stability of the \texttt{ValueTest}.
Additionally, we want to investigate using \enquote{pre-sampling}, i.e.\ gathering some samples without directly using them for inference, to influence some parameters such as confidence distribution, the prior of the \texttt{BetaTest}, or identifying which inference method could be most suitable.
Finally, we want to extend those techniques to a dynamic setting, where transition probabilities are subject to change~\cite{henzinger2025alignment}.

\textbf{Data availability:} Our implementation, all benchmarks, and scripts to reproduce our results are included in our artifact, which will be publicly available.

\newpage

\bibliographystyle{splncs04}
\bibliography{main}

\newpage
\appendix

\crefalias{section}{appendix}
\crefalias{subsection}{appendix}
\section{Assumptions on Sampling Access and Knowledge}
\label{app:assumptions}
In this appendix, we describe the different settings for the MDP-SMC problem that have appeared in the literature.
The focus is on how exactly we can access an unknown MDP.
There are two independent choices to make: firstly, how we can sample transitions from the MDP (related to \stepsample) and secondly what kind of knowledge we have about the transition function (related to \stepinfer).
For both choices, we first classify settings appearing in the literature, ordering them from strongest to weakest assumption, and then discuss them.

\subsection{Obtaining Samples of Transitions}

\begin{itemize}
    \item Every state, e.g.~\cite{DBLP:conf/cav/YounesS02,DBLP:conf/rss/FuT14,DBLP:conf/colt/AgarwalKY20}: In this setting, we can start a simulation in every state. This is equivalent to freely choosing a state-action pair to sample in every step.
    \item Simulations, e.g.~\cite{DBLP:conf/cav/SenVA04,atva14,AKW19,BazilleGJS20,agarwal2022pac,DBLP:conf/nips/SuilenS0022}: In this setting, we can only start a simulation in the initial state. It then continues by choosing an available action, sampling a successor state and then repeating the choosing and sampling from this state.
    Thus, it is harder to get samples of regions that are only visited with small probability, but the setting is arguably more realistic.
    \item Batch learning, e.g.~\cite{DBLP:books/sp/12/LangeGR12,DBLP:conf/icml/Shi0W0C22,DBLP:conf/ijcai/WienhoftSSDB023,WatO}: In this setting, we have no control over which states are sampled but only get a fixed data set of past interactions.
\end{itemize}
Note that in the first two settings, we can use guidance heuristics to improve the practical performance.
In the batch learning setting, PAC guarantees require that we can get further batches, and that the unknown policy performing the sampling is fair, so that every transition is sampled infinitely often.
Even if we know nothing about the origins of the batch data set, we can get a probabilistic guarantee that the bounds computed by \stepsolve{} are correct; however, they may be far apart, i.e.\ not give us the precision we aim for.

Finally, we mention that in addition to having the samples, we require knowledge of the initial state $\initialstate$ and the set of available actions $\Actions(s)$ for every state $s$ that has been sampled.
For the latter, observe that a single unseen action can completely change the value of the MDP, which would make even PAC-guarantees impossible.

\subsection{Knowledge of the Transition Function}
\begin{itemize}
    \item White box, e.g.~\cite{Puterman}: In this fully-informed setting, we know the whole tuple $\MDP = (\States, \Actions, \mdptransitions)$.
    Our whole goal in this work is to avoid assuming this precise knowledge, in particular knowledge of the exact transition probabilities $\mdptransitions$.
    \item Grey box, e.g.~\cite{DBLP:conf/kbse/HeJBGW10,DBLP:conf/sbmf/YounesCZ10,AKW19,BazilleGJS20,agarwal2022pac,DBLP:conf/nips/SuilenS0022,DBLP:conf/nfm/BaierDWK23,DBLP:conf/aaai/BadingsRA023}: In this setting, we do not know the exact transition probabilities $\mdptransitions(s,a,\successor)$ for any state-action-successor triple.
    However, we know the topology of the underlying graph, i.e.\ for every state-action pair $(s,a)$, we know its successors $\{\successor \mid \mdptransitions(s,a,\successor)>0\}$.
    We discuss variants of this setting below.
    \item Black box, e.g.~\cite{DBLP:journals/theoretics/BrazdilCCFKKM0U25,DacaHenzingerKretinskyPetrov2016,AKW19,agarwal2022pac}: In this setting, we know neither the exact transition probabilities nor the topology of the underlying graph. However, we have a lower bound on the smallest transition probability occurring in the MDP $p_{\min} \leq \min(\{\mdptransitions(s,a,\successor) \mid s, \successor \in\States, a\in\Actions(s)\})$.
    \item No-knowledge: In this setting, we know nothing about the transition function.
\end{itemize}
For completeness, we mention other assumptions appearing in the literature:
We can restrict attention to systems that are ergodic under any pair of strategies~\cite{DBLP:journals/ai/BrafmanT00,DBLP:conf/atva/RabihP09}, essentially requiring that we are in a single strongly connected component; or we can assume knowledge of the mixing time of the MDP~\cite{DBLP:conf/rss/FuT14}, which however is as hard to compute as the value.
Thus, we consider both of these assumptions not fitting for our framework.

\paragraph{Alternative grey box formulations.}
There exist three variants of the greybox setting in the literature.
The first version~\cite{DBLP:conf/kbse/HeJBGW10,DBLP:conf/sbmf/YounesCZ10,DBLP:conf/rss/FuT14} requires immediate knowledge of the full graph.
Equivalently~\cite{BazilleGJS20,DBLP:conf/nips/SuilenS0022,DBLP:conf/nfm/BaierDWK23,DBLP:conf/aaai/BadingsRA023}, we can assume that we have an oracle that for every state gives us the available actions and for every state-action pair $(s,a)$ returns the set of its successors $\{\successor \mid \mdptransitions(s,a,\successor)>0\}$.
By exploring the whole graph, we arrive in the same situation as knowing the whole topology.

A weaker variant of the grey box setting is used in~\cite{AKW19,agarwal2022pac}.
Instead of knowing the the exact set of successors for every state-action pair, we only assume to know the number of successors, formally $\abs{\{\successor \mid \mdptransitions(s,a,\successor)>0\}}$.
This allows us to know that all successors have been sampled as soon as the number of sampled successors is equal the known number of successors.
The assumption is only eventually equivalent to knowing the full topology because we need to wait until we have samples seen every successor.
In particular, we can be unlucky and not sample some successor arbitrarily long, albeit with probability 0. 
As long as we have not seen all successors of a state-action pair, applying qualitative graph algorithms (e.g.\ to detect that a state has no path to the target) is not possible.
Thus the qualitative graph algorithms can be applied immediately for the stronger setting, but only after sampling every transition at least once in the weaker one.

Finally, we mention that~\cite{DBLP:conf/nips/SuilenS0022} requires a combination of both black box and grey box, because in addition to the topology it knows for each transition a lower bound strictly greater than 0.

\paragraph{Reducing Black Box to Grey Box.}
Given that we know $p_{\min}$, we can reduce black box to grey box with enough samples, since then we can be sufficiently certain that we have not overlooked a transition, see~\cite[Remark~2]{WatO}.
Intuitively, the idea is to commit, for example, half the confidence towards inferring the topology.
Assuming we have some confidence dedicated to a state-action pair $(s, a)$, we can estimate lower bounds on the transition probabilities $\underline{p}(\successor)$ for each seen successor $\successor$ by using e.g.\ our presented confidence sequence on probabilities.
Once the inferred lower bounds are large enough, i.e.\ $\sum_{\successor} \underline{p} \geq 1 - p_{\min}$, we know that we have seen all successors of this state-action pair with high confidence.
After this happened for every state-action pair we encountered, we can conclude that we have \enquote{seen} the entire system, can apply graph-based analysis, and switch to the grey box approach, using the remaining \enquote{unspent} confidence.
Note that since in some cases we might not know the number of states in advance, we would need to apply a similar \enquote{confidence spending} approach (e.g.\ dedicate $\error / 2$ error budget to the first 1000 states, $\error / 4$ to the next 1000, and so on).
While this process may seem excessive, observe that we need to perform graph-based analysis to even identify sink states, i.e.\ those states which can never reach the goal.
In particular, knowing neither the graph nor $p_{\min}$, we cannot provide PAC guarantees:
Intuitively, just missing a single transition (with arbitrarily small probability) can change the value from 0 to 1, and without any further information we cannot prove that we have seen \enquote{everything} just from finitely many finite samples.

\paragraph{Discussion.}
The black box assumption is a relatively light assumption in many realistic scenarios. For instance, \enquote{bounds on the rates for reaction kinetics in chemical
reaction systems are typically known; for models in the PRISM language~\cite{kwiatkowska2011prism}, the bounds can be easily inferred without constructing the respective state-space}~\cite[p.~2-3]{DacaHenzingerKretinskyPetrov2016}.

We believe that the grey box setting has slightly stronger assumptions than the black box setting, as grey box has topological information about every state, while black box only has a global lower bound. However, one can also argue that the grey box setting is more realistic because of \enquote{systems for which probabilities are governed by uncertain forces (e.g.\ error rates): in this case, it is not easy to have a lower bound on the minimal transition probability, but we can assume that the set of transitions is known}~\cite[p.~20]{BazilleGJS20}.

\paragraph{Observability.}
One assumption that is common to all approaches mentioned above is \emph{observability}, i.e.\ that at any time we can exactly see which state the system is in.
In some cases, the system as a whole is a black box and may even just present us with currently available actions and feedback on the state of the objective (e.g.\ has the goal been reached or not).
Here, none of the mentioned approaches (including ours) is applicable.
In general, \emph{partially observable} MDPs are undecidable with infinite-horizon objectives, even when the underlying structure is known~\cite{MadaniHC99}.

\section{Reusing Confidence is Incorrect} \label{app:sota_bad}

As mentioned, most approaches following the \stepsample{}-\stepinfer{}-\stepsolve{} framework, e.g.\ \cite{WeiningerGMK21,agarwal2022pac,DBLP:conf/nips/SuilenS0022,BaiDubWie23,DBLP:journals/jair/BadingsRAPPSJ23,DBLP:journals/corr/abs-2310-12248}, are not designed with a fixed number of samples in mind, but rather sample continuously until the property of interest can be estimated with the given precision $\varepsilon$.
They also re-use the entire confidence in each \stepinfer{}, intuitively assuming that since we terminate only once the precision is achieved, the correctness of all previous \stepinfer{}s is irrelevant, i.e.\ that $\Pr[\text{error} \mid \text{termination}] = \Pr[\text{error}]$, which is not the case in general.

Concretely, the (flawed) correctness proofs for these sequential procedures usually relies on the soundness of the statistical method used in each \stepsample{}-\stepinfer{} cycle.
While sound inference methods such as Clopper-Pearson or Hoeffding indeed ensure that each single \stepinfer{} step is correct with at least the nominal probability $\gamma$, this is not sufficient to show correctness of the sequential algorithm since the termination condition might be \emph{biased}.

As given by \cref{lemma:lil}, there are stopping times for which the failure condition will happen with probability $\geq \error$.
A more detailed analysis, weighting the precise probability of a failure for a given stopping time with the probability of observing each stopping time, reveals that the sequential algorithm using Clopper-Pearson is unsound for certain parameters~\cite{Frey10}.
Note that this can be cast as simple MDP, thus already showing that the underlying logic of re-using confidence for multiple \stepinfer{} steps for a generic fixed-time confidence interval is incorrect.
We emphasize that therefore this existing work already provides a counterexample to the correctness proofs, as Clopper-Pearson is a correct fixed-time confidence interval, and the correctness proofs of the cited works only rely on this fact.

In practice, we could not find a case where the concrete implementations using the Hoeffding / Okamoto bound fail reliably, as this bound is extremely conservative.
(Note that even if these methods would turn out to be sound, they are far less sample-efficient, as they rely on such a conservative bound.)
We however provide a generic example that shows that fixed time confidence intervals are ill suited for the statistical model checking problem posed in~\cref{subsec:statistical_model_cheking}.

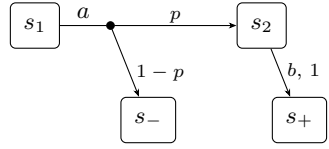
\begin{figure}[t]
    \centering
    \begin{tikzpicture}[auto,state/.append style={minimum height=0.6cm,minimum width=0.6cm}]
        \node[state] at (0,0)  (s1) {$s_1$};
        \node[state] at (1.5,-1.25) (sink) {$s_-$};
        \node[state] at (3,0)  (s2) {$s_2$};
        \node[state] at (3.5,-1.25) (goal) {$s_+$};
        \node[actionnode] at (1,0) (a) {};
        \path[actionedge]
            (s1) edge node[action] {$a$} (a)
        ;
        \path[probedge]
            (a) edge[pos=0.8] node[prob] {$1 - p$} (sink)
            (a) edge node[prob] {$p$} (s2)
        ;
        \path[directedge]
            (s2) edge[pos=0.7] node[prob] {$b$, $1$} (goal)
        ;
    \end{tikzpicture}
    \caption{
        Example MDP to show the problem with re-using confidence budget. 
    }
    \label{fig:conf_example}
\end{figure}
\begin{example} \label{ex:conf_example}
Consider the MDP on $\States=\{s_1,s_2,s_+,s_-\}$ with initial state $s_1$, target state $s_+$, sink state $s_-$, and transitions $\mdptransitions(s_1,a,s_2)=p$, $\mdptransitions(s_1,a,s_-)=1 - p$, and $\mdptransitions(s_2,b,s_+)=1$, where $p\in(0,1)$ (see \cref{fig:conf_example}).
Then $\Probability[\Diamond s_+]=p_\Diamond=p$.
The SMC algorithm estimates $p_\Diamond$ by combining Hoeffding confidence intervals for $\mdptransitions(s_1,a,s_2)$ and $\mdptransitions(s_2,b,s_+)$ via interval arithmetic and a union bound.
For $i\in\{1,2\}$, let $I^i_{\sampletime_i}=[L^i_{\sampletime_i},U^i_{\sampletime_i}]$ be a confidence interval for $p_i$ with error probability $\error/2$.
Assume the sampling strategy keeps $\sampletime_1$ so small that $I^1_{\sampletime_1}=[0,1]$, while $\sampletime_2\to\infty$.
Then, the resulting interval for $p_\Diamond$ is $I^\Diamond_{\sampletime_1,\sampletime_2}\coloneqq I^1_{\sampletime_1}\cdot I^2_{\sampletime_2}=[0,U^1_{\sampletime_2}]$.
If $\varepsilon=p$, then the stopping time $\tau\coloneqq \inf\{\sampletime_2\in\PosNaturals \mid |I^\Diamond_{\sampletime_1,\sampletime_2}|<\varepsilon\}$ coincides with $\tau=\inf\{\sampletime_1\in\PosNaturals\mid U^1_{\sampletime_1}<p\}$. By the law of the iterated logarithm~\cite[p.~377]{grimmett2020probability}, this event occurs almost surely, implying that the claimed guarantee fails, i.e.\ $\Probability[p_\Diamond\in I^\Diamond_{\sampletime_1,\tau}]=\Probability[p_2\leq U^2_\tau]=0\neq 1-\error$.
Note that grey box algorithms could infer that $\mdptransitions(s_2,b,s_+) = 1$ by graph analysis alone.
By, for example, changing from $1$ to $1 - \beta$ with $\beta \ll \varepsilon$ we can prevent this, too, however at the cost of simplicity of the example.
\end{example}

\section{Additional Definitions}
\label{app:defs}%
Here, we provide formal definitions of existing confidence intervals and sequences in our notation.
In the following, define the the empirical and predictive empirical variance as 
$
    \variance_{\successor}(z_{1:\sampletime}) \coloneqq \tfrac{1}{\sampletime-1}{\sum}_{i=1}^{\sampletime}( \indicator{z_i = \successor} - \average_{\successor}(z_{1:\sampletime}) )^2$ and
$   \overrightarrow{\variance}_\successor(z_{1:\sampletime}) \coloneqq \tfrac{1}{\sampletime}{\sum}_{i=1}^{\sampletime}( \indicator{z_i=\successor} - \average_{\successor}(z_{1:i-1}) )^2$, respectively.
For $\sampletime=0$ we define $z_{1:\sampletime}$ as the empty sequence $\epsilon$, set $\average_{\successor}(\epsilon)=1/2$, $ \overrightarrow{\variance}_\successor(\epsilon) = 1/4$, and set $\variance_{\successor}(z_{1:\sampletime}) =1/4$ for $\sampletime \in \{0,1\}$.

\subsection{Rate Optimality} \label{app:defs:rate_optimal}

\paragraph{Bernoulli Confidence Intervals.}
We define $\mathfrak{I}$ as the set of all confidence intervals for the Bernoulli mean. Each confidence interval $\confinter\in \mathfrak{I}$ satisfies for each error probability $\error\in (0,1)$ that
\begin{align}
    \forall \sampletime\in \PosNaturals \colon \inf_{p\in [0,1]} \Probability_p\left[ p \in \confinter(\error; X_{1:\sampletime}) \right] \geq 1-\error.
\end{align}

\paragraph{Empirically Anchored.}
The confidence interval $\confinter\in \mathfrak{I}$ is empirically anchored, if it always contains the empirical average.
Formally, $\confinter$ is empirically anchored if for all $\sampletime\in \PosNaturals$, all $\error\in (0,1)$, and all $x_{1:\sampletime}\in \{0,1\}^{\sampletime}$,
\begin{align*}
    \frac{1}{\sampletime}\sum_{i=1}^{\sampletime} x_i \in \confinter(\error; x_{1:\sampletime}).
\end{align*}

\paragraph{Interval Width.}
The width of the confidence interval $\confinter\in \mathfrak{I}$ for a given sequence of samples $x_{1:\sampletime}\in \{0,1\}^\sampletime$ and error probability $\error\in (0,1)$ is
\begin{align}
  |\confinter(\error; x_{1:\sampletime})| \coloneqq \sup \confinter(\error; x_{1:\sampletime}) - \inf \confinter(\error; x_{1:\sampletime}).
\end{align}
The maximal width after $\sampletime\in \PosNaturals$ of the confidence interval $\confinter\in \mathfrak{I}$ is
\begin{align}
    W_\confinter(\error, \sampletime) \coloneqq \sup_{x_{1:\sampletime}\in \{0,1\}^{\sampletime}} |\confinter(\error; x_{1:\sampletime})|.
\end{align}
Intuitively, this is the worst case width that may occur on a sequence of Bernoulli random variables. 
The optimal maximal width among $\mathfrak{I}$ is given by
\begin{align}
    W^\star(\error, \sampletime) \coloneqq \inf_{\confinter' \in \mathfrak{I}} W_{\confinter'}(\error, \sampletime).
\end{align}

\begin{restatable}{lemma}{optimalWidth}
\label{lemma:optimal-width-rate}
The optimal maximal width benchmark for Bernoulli confidence intervals is of order
\begin{align}
    W^\star(\error, \sampletime)
    =
    \Theta\left(
        \min\left\{1,\sqrt{\frac{\log(1/\error)}{\sampletime}}\right\}
    \right).
\end{align}
\end{restatable}
\noindent
Since we are interested in the shrinking-width regime, we formulate rate
optimality asymptotically, i.e.\ for sufficiently large sample sizes and
sufficiently small error probabilities~\cite{howard2020time,waudby2024estimating}.
Related to the discussion of minimax optimal width, where optimality is defined w.r.t. to the confidence intervals expected behaviour~\cite{BroCaiDas01,Tsybakov2009}, we define optimality w.r.t. the worst case behaviour.
This allows us to address the common fixed time confidence intervals uniformly. 
\begin{definition}
\label{def:width}
A confidence interval $\confinter\in \mathfrak{I}$ has asymptotically \emph{rate-optimal maximal width}, if there exist constants
$0<c_0\le C_0<\infty$, $\sampletime_0\in\PosNaturals$, and
$\error_0\in(0,1)$ such that for all $\sampletime\ge \sampletime_0$ and all
$\error\in(0,\error_0]$ we have
\begin{equation}
\label{eq:width}
c_0 \cdot
\min\left\{1,\sqrt{\frac{\log(1/\error)}{\sampletime}}\right\}
\le
W_\confinter(\error,\sampletime)
\le
C_0 \cdot
\min\left\{1,\sqrt{\frac{\log(1/\error)}{\sampletime}}\right\}.
\end{equation}
\end{definition}

\subsection{Confidence Intervals} \label{app:defs:intervals}

We provide definitions of the classical Hoeffding~\cite{Hoe63}, (empirical) Bernstein~\cite{MauPon09}, and Clopper-Pearson intervals~\cite{CloPea34}:
\begin{align*}
    \confinter_{\successor}^H(\error; z_{1:\sampletime})
    &=
    \left[\average_{\successor}(z_{1:\sampletime}) \pm \sqrt{\log(2/\error) / (2 \sampletime)}\right]\cap[0,1],\\
    \confinter_{\successor}^B(\error;z_{1:\sampletime})
    &=
    \left[\average_{\successor}(z_{1:\sampletime}) \pm \sqrt{2\variance_{\successor}(z_{1:\sampletime})\log(3/\error) / \sampletime} + 7\log(3/\error) / (3(\sampletime-1))\right]\cap[0,1],\\
    \confinter_{\successor}^C(\error;z_{1:\sampletime})
    &=
    \Big[
      \mathrm{Beta}^{-1}(\error/2; k_\sampletime, \sampletime - k_\sampletime + 1),\
      \mathrm{Beta}^{-1}(1-\error/2; k_\sampletime + 1, \sampletime - k_\sampletime)
    \Big],
\end{align*}
where $k_\sampletime = \successes_\successor(z_{1:\sampletime})$ and $\mathrm{Beta}^{-1}$ the inverse of the Beta distribution. Whenever an interval is undefined, e.g., $\confinter_{\successor}^B$ for $t<2$ and $\confinter_{\successor}^C$ for $k_\sampletime=0$ or $k_\sampletime=\sampletime$, they are set to the trivial interval $[0,1]$.
All three intervals are empirically anchored and their maximal widths are of order $\Theta(\sqrt{\log(1/\error)/\sampletime})$ (see also \cref{lemma:cp-rate-optimal,lemma:empbern-rate-optimal}).

\subsection{Stitching Sequences} \label{app:defs:stitched}
Fix $c>1$ as the parameter of an exponential grid $N_{\mathrm{exp}}^{1,c}$, an initial time $\sampletime_0\in \PosNaturals$, and let $h$ be a confidence-spending function.
Define $k_1 = (c^{1/4}+c^{-1/4})/\sqrt{2}$, $k_2=(\sqrt{c}+1)/2$,
$
    m(v) \coloneqq \left\lceil \log(v/\sampletime_0) / \log(c) \right\rceil$,
and
$
g(\error; v) \coloneqq \log h\left(m(v)\right) + \log\left(2/\error\right)$.
Then, for every $\sampletime \in \PosNaturals$, $\error\in (0,1)$, $\successor \in \Successors$, and $z_{1:\sampletime}\in \Successors^\sampletime$, the \emph{Stitched Hoeffding} confidence sequence~\cite{howard2021time} is
\begin{align*}
  \confseq_\successor^{SH}(\error;z_{1:\sampletime})
  =
  \average_\successor(z_{1:\sampletime})
   \pm
  \sqrt{ \tfrac{1}{\sampletime} k_1^2  g(\error;\sampletime)} \cap[0,1],
\end{align*}
and the \emph{Stitched Bernstein} confidence sequence~\cite{howard2021time} is
\begin{align*}
  \confseq_\successor^{SB}(\error;z_{1:\sampletime})
  =
  \average_\successor(z_{1:\sampletime})
   \pm
  \tfrac{1}{\sampletime}\left( \sqrt{k_1^2 v_\sampletime g(\error; v_\sampletime)\ +\ k_2^2  g(\error; v_\sampletime)^2}\ +\ k_2  g(\error; v_\sampletime) \right) \cap[0,1],
\end{align*}
where $v_\sampletime=\max(1, \sampletime \cdot \overrightarrow{\variance}_\successor(z_{1:\sampletime}))$ 

\subsection{Betting Sequences} \label{app:defs:betting}
For $\sampletime \in \PosNaturals$, $\error\in(0,1)$, $\successor \in \Successors$, and $z_{1:\sampletime}\in\Successors^\sampletime$, define the \emph{Betting Hoeffding} confidence sequence~\cite{waudby2024estimating}
\begin{align*}
  \confseq_\successor^{BH}(\error; z_{1:\sampletime})
  =
  \frac{{\sum}_{i=1}^\sampletime \lambda_i \indicator{z_i=\successor}}{{\sum}_{i=1}^\sampletime \lambda_i}
   \pm
  \frac{\log(2/\error) + {\sum}_{i=1}^\sampletime \lambda_i^2/8}{{{\sum}_{i=1}^\sampletime \lambda_i}},
\end{align*}
where $\lambda_i=\min(1, \eta_i)$ and $\eta_i = \sqrt{(8\log(2/\error))/(i\log(i+1))}$.
Moreover, the \emph{Betting Bernstein} confidence sequence~\cite{waudby2024estimating} is defined as
\begin{align*}
  &\confseq_\successor^{BB}(\error; z_{1:\sampletime})
  =
  \frac{\sum_{i=1}^\sampletime \lambda_i \indicator{z_i=\successor}}{\sum_{i=1}^\sampletime \lambda_i}
   \pm
  \frac{\log(2/\error) + \sum_{i=1}^\sampletime v_i\left(-\log(1-\lambda_i)-\lambda_i\right)}{\sum_{i=1}^\sampletime \lambda_i} \cap[0,1],\\
  &\text{where}\quad v_i=\left(\indicator{z_i=\successor}-\average_\successor(z_{1:i-1})\right)^2,
 \text{ and }
  \lambda_i = \min\left(\frac{1}{2}, \sqrt{\frac{2\log (2/\error)}{\overrightarrow{\variance}_\successor(z_{1:i-1}) i \log (i+1)}} \right).
\end{align*}

\section{Proofs}
\label{app:proof}

Throughout this section, fix a state-action pair $s \in \States$, $a \in \stateactions(s)$ and a successor state $\successor\in\Successors$.
Let $\saprob = \mdptransitions(s, a)$, $Z=(Z_\sampletime)_{\sampletime\in\PosNaturals}$ with $Z_\sampletime\sim\saprob$ i.i.d., and define
\begin{align*}
X_\sampletime \coloneqq \indicator{Z_\sampletime=\successor}\in\{0,1\},\qquad
p \coloneqq \saprob(\successor)\in[0,1].
\end{align*}
Write $\mathcal{F}_\sampletime\coloneqq \sigma(Z_{1:\sampletime})=\variance(X_{1:\sampletime})$.

\optimalWidth*
\begin{proof}
We first prove the upper bound.
Because the Hoeffding interval is data-agnostic the claim follows immediatly, i.e., for  $\error<1/2$
\begin{align*}
W^\star(\error,\sampletime) \leq W_{\confinter^H}(\error,\sampletime)
\leq
\min\left\{1,\sqrt{\frac{2\log(2/\error)}{\sampletime}}\right\}\leq C
\min\left\{1,\sqrt{\frac{2\log(1/\error)}{\sampletime}}\right\}.
\end{align*}
for some constant $C>0$ because $W^\star$ is the infimum over all Bernoulli confidence intervals.
We now prove the lower bound.
Fix an arbitrary Bernoulli confidence interval $\confinter\in\mathfrak I$ and abbreviate $w\coloneqq W_{\confinter}(\error,\sampletime).$
If $w> 1/4$, then trivially
\begin{align*}
\frac14\min\left\{1,\sqrt{\frac{\log(1/\error)}{\sampletime}}\right\} \leq w.
\end{align*}
It therefore remains to consider the case $w\leq \tfrac{1}{4}$. We show this using a standard argument from information theory, i.e., we derive a Le Cam-style lower bound, which reduces this problem to binary hypothesis testing. Towards that end let $p_-\coloneqq \frac12-w$ and $p_+\coloneqq \frac12+w$.
Notice that since $p_+-p_-=2w>w$, no interval of width at most $w$ can contain both $p_-$ and $p_+$ simultaneously. We construct a hypothesis test trying to distinguish between those two parameters. 
Specifically, we test 
\begin{align*}
\mathcal{H}_0:p=p_-  \quad \text{against} \quad \mathcal{H}_1:p=p_+ \quad \text{by} \quad  \psi(x_{1:\sampletime})
\coloneqq
\indicator{p_+\in \confinter(\error;x_{1:\sampletime})}.
\end{align*}
Because $\confinter$ has maximal width at most $w$, the event
$\{p_+\in \confinter(\error;x_{1:\sampletime})\}$ implies
$p_-\notin \confinter(\error;x_{1:\sampletime})$.
Hence, by fixed-time validity of $\confinter$ we bound the false positive probability (type-I error) and the false negative probability (type-II error) respectively by $\error$, i.e.  
\begin{equation*}
\Probability_{p_-}(\psi=1)
\leq
\Probability_{p_-}\left(p_-\notin \confinter(\error;X_{1:\sampletime})\right)
\leq
\error,
\quad 
\Probability_{p_+}(\psi=0)
=
\Probability_{p_+}\left(p_+\notin \confinter(\error;X_{1:\sampletime})\right)
\leq
\error.
\end{equation*}
Thus the type-I and type-II errors satisfy $\alpha+\beta\leq 2\error.$
On the other hand, by the Bretagnolle-Huber inequality~\cite[p.~190]{lattimore2020bandit} we obtain
\begin{equation*}
\tfrac{1}{2} \exp\left(-\sampletime\,\mathrm{KL}(p_-\|p_+)\right)  \leq \alpha+\beta \leq 2\error
\quad \text{which is} \quad
\log\left(1/(4\error)\right) \leq  \sampletime\,\mathrm{KL}(p_-\|p_+).
\end{equation*}
For $w\leq \tfrac14$ we compute
\begin{equation*}
\mathrm{KL}(p_-\|p_+)
=
\left(\frac12-w\right)\log\frac{\frac12-w}{\frac12+w}
+
\left(\frac12+w\right)\log\frac{\frac12+w}{\frac12-w}
=
2w\log\frac{1+2w}{1-2w}.
\end{equation*}
Using $\log(1+x)\leq x$ and $-\log(1-x)\leq x/(1-x)$ for $x\in[0,1)$, we obtain
\begin{equation*}
\log\frac{1+2w}{1-2w}
=
\log(1+2w)-\log(1-2w)
\leq
2w+\frac{2w}{1-2w}
\leq
8w,
\end{equation*}
and therefore $\mathrm{KL}(p_-\|p_+)\leq 16w^2$.
Substituting this into the previous bound gives
\begin{align*}
   & \log\left(1/(4\error)\right) \leq  \sampletime \mathrm{KL}(p_-\|p_+) \leq  16\sampletime w^2 \quad 
    \iff \quad \frac{1}{4}\sqrt{\frac{\log(1/(4\error))}{\sampletime}} \leq w
\end{align*}
Notice that for every $\error\leq \tfrac{1}{8}$ 
\begin{align*}
\frac{1}{4}\sqrt{\frac{\log(1/(4\error))}{\sampletime}} = 
\frac{1}{4}\sqrt{\frac{\log(1/\error)-\log 4}{\sampletime}} 
\geq 
\frac{1}{4}\sqrt{\frac{\frac{1}{3}\log(1/\error)}{\sampletime}} =
\frac{1}{4\sqrt3} 
\sqrt{\frac{\log(1/\error)}{\sampletime}}.
\end{align*}
Combining this with the case $w>1/4$ results in 
\begin{align*}
\frac{1}{4\sqrt3}
\min\left\{1,\sqrt{\frac{\log(1/\error)}{\sampletime}}\right\} \leq w.
\end{align*}
Since $\confinter\in\mathfrak I$ was arbitrary we get 
\begin{align*}
\frac{1}{4\sqrt3}
\min\left\{1,\sqrt{\frac{\log(1/\error)}{\sampletime}}\right\} \leq W^\star(\error,\sampletime).
\end{align*}
\end{proof}

\subsection{Union Bound}

\lilLemma*
\begin{proof}[of~\cref{lemma:lil}]
Assume $p=\tfrac{1}{2}$ and let $\confinter_{\successor}$ be an empirically anchored
fixed-time confidence interval with asymptotically rate-optimal maximal width.
By \cref{def:width}, there exists $\error_0\in(0,1)$ such that for every
$\error\in(0,\error_0]$ there exist a constant $C<\infty$ and a sample count
$n_0\in\PosNaturals$ such that for all $n\geq n_0$,
\begin{equation*}
W_{\confinter_{\successor}}(\error,n)
\leq
C\sqrt{\frac{\log(1/\error)}{n}}.
\end{equation*}
Fix $\error\in(0,\error_0]$ and define
\begin{equation*}
A_\error
\coloneqq
\left\{
\forall \sampletime\in\PosNaturals:\ 
p\in \confinter_{\successor}(\error;Z_{1:\sampletime})
\right\} \subseteq \left\{
\forall n\geq n_0:\ 
p\in \confinter_{\successor}(\error;Z_{1:n})
\right\}.
\end{equation*}
Since it suffices to show that
\begin{equation*}
\Probability\left[
\forall n\geq n_0:\ 
p\in \confinter_{\successor}(\error;Z_{1:n})
\right]
=
0.
\end{equation*}
Because $\confinter_{\successor}$ is empirically anchored, for every $n\geq n_0$
on the event $\{p\in \confinter_{\successor}(\error;Z_{1:n})\}$ we have
\begin{equation*}
\big|\average_\successor(Z_{1:n})-p\big|
\leq
\big|\confinter_{\successor}(\error;Z_{1:n})\big|
\leq
W_{\confinter_{\successor}}(\error,n)
\leq
C\sqrt{\frac{\log(1/\error)}{n}}.
\end{equation*}
Hence
\begin{equation*}
\left\{
\forall n\geq n_0:\ 
p\in \confinter_{\successor}(\error;Z_{1:n})
\right\}
\subseteq
\left\{
\forall n\geq n_0:\ 
\big|\average_\successor(Z_{1:n})-p\big|
\leq
C\sqrt{\frac{\log(1/\error)}{n}}
\right\}.
\end{equation*}
Let $X_i\coloneqq \indicator{Z_i=\successor}$ and
$Y_i\coloneqq 2(p-X_i)\in\{-1,+1\}$.
Under $p=\tfrac12$, the variables $(Y_i)_{i\in\PosNaturals}$ are i.i.d.\ with
$\expect[Y_i]=0$ and $\Variance(Y_i)=1$.
Moreover, for every $n\in\PosNaturals$,
\begin{equation*}
\left|np-\successes_\successor(Z_{1:n})\right|
=
\frac12\left|\sum_{i=1}^n Y_i\right|
\quad \text{and thus} \quad \left|\average_\successor(Z_{1:n})-p\right|
=
\frac{1}{2n}\left|\sum_{i=1}^n Y_i\right|.
\end{equation*}
By the law of the iterated logarithm for i.i.d.\ mean-zero, variance-one variables,
\begin{equation*}
\limsup_{n\to\infty}
\frac{\left|\sum_{i=1}^n Y_i\right|}{\sqrt{2n\log\log n}}
=
1
\qquad\text{almost surely.}
\end{equation*}
Since $\sqrt{n\log(1/\error)} = o\left(\sqrt{n\log\log n}\right)$, it follows that
\begin{equation*}
\frac{\left|\sum_{i=1}^n Y_i\right|}{2\sqrt{n\log(1/\error)}}>C \quad \text{or equivalently} \quad  \big|\average_\successor(Z_{1:n})-p\big|
>
C\sqrt{\frac{\log(1/\error)}{n}}
\end{equation*}
occurs infinitely often almost surely.
In particular,
\begin{equation*}
\Probability\left[
\forall n\geq n_0:\ 
\big|\average_\successor(Z_{1:n})-p\big|
\leq
C\sqrt{\frac{\log(1/\error)}{n}}
\right]
=
0.
\end{equation*}
By the earlier inclusion, this implies that $\Probability[A_\error]=0$.
\end{proof}

\unionGridThm*
\begin{proof}[of~\cref{thrm:sound_union}]
Let $N=(n_i)_{i\in\PosNaturals}$ be an increasing grid and $h$ satisfy $\sum_{i=1}^\infty 1/h(i)\leq 1$.
For each $i\in\PosNaturals$, define the \enquote{failure event}
\begin{equation*}
E_i \coloneqq \left\lbrace p\notin \confinter_{\successor}\left(\error/h(i); Z_{1:n_i}\right)\right\rbrace .
\end{equation*}
By fixed-time validity of $\confinter_{\successor}$ with level $\error/h(i)$ at sample size $n_i$,
\begin{equation*}
\Probability[E_i]\ \leq \error/h(i)\qquad\text{for all }i\in\PosNaturals.
\end{equation*}
We now take a union bound over the grid and because $\sum_i 1/h(i)\leq 1$ we get
\begin{align*}
\Probability\left[ \exists i\in N \colon p\notin \confinter_{\successor}\left(\error/h(i); Z_{1:n_i}\right)\right]
&=\Probability\left[ \bigcup_{i=1}^\infty E_i\right]\\
&\leq \sum_{i=1}^\infty \Probability[E_i]
\leq \sum_{i=1}^\infty \error/h(i)
\leq \error.
\end{align*}
\end{proof}

\unionConvergenceLemma*
\begin{proof}[of~\cref{lemma:union_convergence}]
In the following, we write $\asymp$ to denote (asymptotic) equality up to constants.
Assume that $\confinter_{\successor}$ has rate-optimal maximal width, i.e.
\begin{align*}
1\geq W_{\confinter_{\successor}}(\error,n_i)
\asymp
\min\left\lbrace1, \sqrt{\frac{\log(1/\error)}{n_i}} \right\rbrace,
\end{align*}
uniformly over $\error\in(0,1)$ and $i\in\PosNaturals$.
Since we clip the interval and are interested in the asymptotic rate we focus on the non-constant term only.
On any grid, plugging in $\error=\error/h(i)$ gives
\begin{align*}
W_{\confinter_{\successor}}(\error/h(i),n_i)
\asymp
\sqrt{\frac{\log(h(i)/\error)}{n_i}}.
\end{align*}

\noindent
\emph{Polynomial grid $N_{\mathrm{poly}}^{b,c}=(\lfloor b i^c\rfloor)_i$.}
Ignoring floors at the $\asymp$ level, $n_i\asymp b i^c$, hence $i\asymp (n_i/b)^{1/c}$ and $\log i = \tfrac1c\log(n_i/b)+O(1)$.
Thus
\begin{align*}
\log\left(\tfrac{h_{\mathrm{poly}}^a(i)}{\error}\right)
= a\log i + \log(\eta_p/\error)
= \tfrac{a}{c}\log(n_i/b)+\log(\eta_p/\error)+O(1),
\end{align*}
and
\begin{align*}
\log\left(\tfrac{h_{\mathrm{exp}}^a(i)}{\error}\right)
= i\log a + \log(\eta_e/\error)
= (n_i/b)^{1/c}\log a + \log(\eta_e/\error) + O(1).
\end{align*}
Substituting into $\sqrt{\log(h(i)/\error)/n_i}$ results in the two claimed expressions up to multiplicative constants.

\noindent
\emph{Exponential grid $N_{\mathrm{exp}}^{b,c}=(\lfloor b c^i\rfloor)_i$.}
Again $n_i\asymp b c^i$, hence $i=\log_c(n_i/b)+O(1)$ and
$\log i = \log\log_c(n_i/b)+O(1)$.
Therefore
\begin{align*}
\log\left(\tfrac{h_{\mathrm{poly}}^a(i)}{\error}\right)
= a\log i + \log(\eta_p/\error)
= a\log\log_c(n_i/b)+\log(\eta_p/\error)+O(1),
\end{align*}
and
\begin{align*}
\log\left(\tfrac{h_{\mathrm{exp}}^a(i)}{\error}\right)
= i\log a + \log(\eta_e/\error)
= \log_c(n_i/b)\log a + \log(\eta_e/\error)+O(1).
\end{align*}
Substituting finishes the proof (and the $N_{\mathrm{poly}}^{b,1}$ row is the special case $c=1$).
\end{proof}

\polyDominanceLemma*
\begin{proof}[of~\cref{lemma:poly_dominance}]
Let $\confinter_{\successor}$ have rate-optimal maximal width.
Then there exist constants $0<c_0\leq C_0<\infty$, a sample count $n_0\in\PosNaturals$, and an error probability $\error_0\in(0,1)$ such that for all $n\geq n_0$ and all $\error\in(0,\error_0]$,
\begin{align*}
c_0\sqrt{\frac{\log(1/\error)}{n}}
\leq
W_{\confinter_{\successor}}(\error,n)
\leq
C_0\sqrt{\frac{\log(1/\error)}{n}} \leq 1.
\end{align*}
Fix a grid $N=(n_i)$ and parameters $a,b>1$.
For polynomial spending $h_{\mathrm{poly}}^a(i)=\eta_p i^a$ and exponential spending $h_{\mathrm{exp}}^b(i)=\eta_e b^i$, we have
\begin{align*}
\log\left(\tfrac{h_{\mathrm{poly}}^a(i)}{\error}\right)
=
a\log i + \log(\eta_p/\error),
\qquad
\log\left(\tfrac{h_{\mathrm{exp}}^b(i)}{\error}\right)
=
i\log b + \log(\eta_e/\error).
\end{align*}
Since $a\log i = o(i)$, there exists $i_0\in\PosNaturals$ such that for all $i\geq i_0$,
\begin{align*}
C_0^2\left(a\log i + \log(\eta_p/\error)\right)
\leq
c_0^2\left(i\log b + \log(\eta_e/\error)\right).
\end{align*}
Therefore, for all sufficiently large $i$,
\begin{align*}
W_{\confinter_{\successor}}(\error/h_{\mathrm{poly}}^a(i),n_i)
&\leq
C_0\sqrt{\frac{a\log i + \log(\eta_p/\error)}{n_i}}
\\
&\leq
c_0\sqrt{\frac{i\log b + \log(\eta_e/\error)}{n_i}}
\\
&\leq
W_{\confinter_{\successor}}(\error/h_{\mathrm{exp}}^b(i),n_i) \leq 1.
\end{align*}
\end{proof}

\nonConvergence*
\begin{proof}[of~\cref{corr:non_convergence}]
Let $N=(n_i)_{i\in\PosNaturals}$ be a grid with constant spacing.
Then $n_i=\Theta(i)$.
Hence, by \cref{lemma:union_convergence}, under exponential spending we obtain
\begin{align*}
W_{\confinter_{\successor}}(\error/h_{\mathrm{exp}}^a(i),n_i)
=
\Theta(1).
\end{align*}
In particular, the maximal widths do not converge to $0$.
\end{proof}

\begin{lemma}
\label{lemma:empbern-rate-optimal}
The empirical Bernstein interval $\confinter_{\successor}^B$ (\cref{app:defs:intervals}) has rate-optimal maximal width.
\end{lemma}

\begin{proof}
Since $\confinter_{\successor}^B$ is a Bernoulli confidence interval, we have $W_{\confinter_{\successor}^B}(\error,\sampletime)\geq W_\star(\error,t)$
for all $\error\in(0,1)$ and $\sampletime\in \PosNaturals$.
By the optimal maximal width convergence rate from \cref{app:defs:rate_optimal}, there exist constants
$c_0>0$, $\sampletime_0\in\PosNaturals$, and $\error_0\in(0,1)$ such that for all
$\sampletime\geq \sampletime_0$ and $\error\in(0,\error_0]$,
\begin{align*}
1\geq W_{\confinter_{\successor}^B}(\error,\sampletime)\geq c_0\sqrt{\frac{\log(1/\error)}{\sampletime}}.
\end{align*}
It remains to prove the matching upper bound.
Fix $\sampletime\in\PosNaturals$, $\error\in(0,1)$, and a sample sequence $z_{1:\sampletime}\in \Successors^\sampletime$.
By definition of $\confinter_{\successor}^B$,
\begin{align*}
|\confinter_{\successor}^B(\error;z_{1:\sampletime})|
&\leq
2\sqrt{\frac{2\variance_{\successor}(z_{1:\sampletime})\log(3/\error)}{\sampletime}}
+\frac{14\log(3/\error)}{3(t-1)}.
\end{align*}
Since $\variance_{\successor}(z_{1:\sampletime})\leq \frac14$ for Bernoulli samples we get
\begin{align*}
|\confinter_{\successor}^B(\error;z_{1:\sampletime})|
\leq
\sqrt{\frac{2\log(3/\error)}{\sampletime}}
+\frac{14\log(3/\error)}{3(t-1)}.
\end{align*}
We now distinguish two cases.
If $\log(3/\error) \geq t$ and $\sampletime\geq 2$, then $2\sqrt{\log(3/\error)/t}\geq 1$ and the confidence interval is trivial.
Otherwise, for $t>1$ we have
\begin{align*}
\frac{\log(3/\error)}{t-1}
\leq
\frac{2\log(3/\error)}{\sampletime}
\leq
2\sqrt{\frac{\log(3/\error)}{\sampletime}},
\end{align*}
and therefore
\begin{align*}
|\confinter_{\successor}^B(\error;z_{1:\sampletime})| &\leq \sqrt{\frac{2\log(3/\error)}{\sampletime}}
+\frac{14\log(3/\error)}{3(t-1)}\\
&\leq \sqrt{\frac{2\log(3/\error)}{\sampletime}}
+\frac{14}{3}2\sqrt{\frac{\log(3/\error)}{\sampletime}}
\\
&\leq 
\left(\sqrt{2}+\frac{28}{3}\right)\sqrt{\frac{\log(3/\error)}{\sampletime}}.
\end{align*}
Taking the supremum over all sample sequences gives
\begin{align*}
W_{\confinter_{\successor}^B}(\error,\sampletime)
\leq
\left(\sqrt{2}+\frac{28}{3}\right)\sqrt{\frac{\log(3/\error)}{\sampletime}}.
\end{align*}
It remains to replace $\log(3/\error)$ by $\log(1/\error)$.
For every $\error\in(0,1)$,
\begin{align*}
\log(3/\error)=\log(1/\error)+\log 3.
\end{align*}
Hence, for every $\error\in(0,\error_0]$,
\begin{align*}
\frac{\log(3/\error)}{\log(1/\error)}
=
1+\frac{\log 3}{\log(1/\error)}
\leq
1+\frac{\log 3}{\log(1/\error_0)}.
\end{align*}
Therefore,
\begin{align*}
\log(3/\error)
\leq
\left(1+\frac{\log 3}{\log(1/\error_0)}\right)\log(1/\error)
\qquad
\text{for all }\error\in(0,\error_0].
\end{align*}
Substituting this into the previous bound
\begin{align*}
W_{\confinter_{\successor}^B}(\error,\sampletime)
&\leq
\left(\sqrt{2}+\frac{28}{3}\right)
\sqrt{1+\frac{\log 3}{\log(1/\error_0)}}
\sqrt{\frac{\log(1/\error)}{\sampletime}}.
\end{align*}
Thus, with
\begin{align*}
C_0
\coloneqq
\left(\sqrt{2}+\frac{28}{3}\right)
\sqrt{1+\frac{\log 3}{\log(1/\error_0)}},
\end{align*}
we obtain for all $\sampletime\geq \max\{2,\sampletime_0\}$ and all $\error\in(0,\error_0]$ that
\begin{align*}
W_{\confinter_{\successor}^B}(\error,\sampletime)
\leq
C_0\sqrt{\frac{\log(1/\error)}{\sampletime}}.
\end{align*}
Together with the lower bound, this proves that $\confinter_{\successor}^B$ has rate-optimal maximal width.
\end{proof}

\begin{lemma}
\label{lemma:cp-rate-optimal}
The Clopper--Pearson interval $\confinter_{\successor}^C$ (\cref{app:defs:intervals}) has rate-optimal maximal width.
\end{lemma}

\begin{proof}
Since $\confinter_{\successor}^C$ is a Bernoulli confidence interval, we have $W_{\confinter_{\successor}^C}(\error,t)\geq W^\star(\error,t)$ for all $\error\in(0,1)$ and $\sampletime\in \PosNaturals$.
By the optimal maximal width convergence rate from \cref{app:defs:rate_optimal}, there exist constants
$c_0>0$, $\sampletime_0\in\PosNaturals$, and $\error_0\in(0,1)$ such that for all
$\sampletime\geq \sampletime_0$ and $\error\in(0,\error_0]$,
\begin{align*}
W_{\confinter_{\successor}^C}(\error,t)\geq c_0\sqrt{\frac{\log(1/\error)}{\sampletime}}.
\end{align*}
It remains to prove the matching upper bound.
We fix $\sampletime\in\PosNaturals$, $\error\in(0,1)$, and $z_{1:\sampletime}\in \Successors^\sampletime$, and denote $k_\sampletime \coloneqq \successes_{\successor}(z_{1:\sampletime})$ and $\hat{p}_\sampletime \coloneqq \average_{\successor}(z_{1:\sampletime})=\frac{k_\sampletime}{\sampletime}$. 
We denote the Clopper-Pearson interval by
\begin{align*}
\confinter_{\successor}^C(\error;z_{1:\sampletime})=[\ell_\sampletime,u_\sampletime] \quad \text{and set} \quad  \varepsilon_\sampletime(\error)\coloneqq \sqrt{\frac{\log(2/\error)}{2t}}.
\end{align*}
We demonstrate the claim by proving that the Clopper-Pearson interval is always contained inside the Hoeffding interval, i.e., 
\begin{align*}
    \hat{p}_\sampletime - \varepsilon_\sampletime(\error) \leq \ell_\sampletime \leq  u_\sampletime  \leq   \hat{p}_\sampletime + \varepsilon_\sampletime(\error).
\end{align*}
We bound the upper endpoint.
If $k_\sampletime=t$, then $u_\sampletime=1=\hat{P}_\sampletime\leq \hat{P}_\sampletime+\varepsilon_\sampletime(\error)$.
Now assume $k_\sampletime<t$.
By definition of the Clopper--Pearson upper endpoint, $u_\sampletime$ is the unique solution of
\begin{align*}
\Probability_{u_\sampletime}\bigl(\mathrm{Bin}(t,u_\sampletime)\leq k_\sampletime\bigr)=\frac{\error}{2}.
\end{align*}
For $p=\hat{P}_\sampletime+\varepsilon_\sampletime(\error)$, Hoeffding's inequality gives
\begin{align*}
\Probability_{p}\bigl(\mathrm{Bin}(t,p)\leq k_\sampletime\bigr)
&=
\Probability_{p}\left(\frac{1}{\sampletime}\mathrm{Bin}(t,p)\leq \hat{P}_\sampletime\right)\\
&\leq
\Probability_{p}\left(\frac{1}{\sampletime}\mathrm{Bin}(t,p)\leq p-\varepsilon_\sampletime(\error)\right)\\
&\leq
\exp\bigl(-2t\varepsilon_\sampletime(\error)^2\bigr)
=
\frac{\error}{2}.
\end{align*}
Since $p\mapsto \Probability_p(\mathrm{Bin}(t,p)\leq k_\sampletime)$ is decreasing, it follows that $u_\sampletime\leq \hat{P}_\sampletime+\varepsilon_\sampletime(\error)$.
The lower bound can be proven analogously. 
Both bounds can be combined to obtain 
\begin{align*}
|\confinter_{\successor}^C(\error;z_{1:\sampletime})|
=
u_\sampletime-\ell_\sampletime
\leq
2\varepsilon_\sampletime(\error)
=
\sqrt{\frac{2\log(2/\error)}{\sampletime}}.
\end{align*}
Taking the supremum over all sample sequences gives
\begin{align*}
W_{\confinter_{\successor}^C}(\error,t)
\leq
\sqrt{\frac{2\log(2/\error)}{\sampletime}}.
\end{align*}
As in~\Cref{lemma:empbern-rate-optimal} we can resolve the $\log(2)$ to obtain our finial result that there exists $C_0<\infty$ such that for all $\sampletime\geq \sampletime_0$ and all $\error\in(0,\error_0]$,
\begin{align*}
W_{\confinter_{\successor}^C}(\error,t)
\leq
C_0\sqrt{\frac{\log(1/\error)}{\sampletime}} \leq 1.
\end{align*}
Together with the lower bound, this proves that $\confinter_{\successor}^C$ has rate-optimal maximal width.
\end{proof}

\subsection{Confidence Sequences} \label{app:proof:confseq}

We additionally provide a formal proof that requiring correctness at all times (in the sense of \cref{eq:conf_set_trans}) is equivalent to correctness at a single \enquote{data-driven} time, i.e.\ a time that is adaptively chosen based on the observations, formally called a \emph{stopping time}.
\begin{restatable}{lemma}{csEquivalenceLemma}
\label{lemma:cs_equivalence}
Let $\error\in (0,1)$ be an error probability, $\confseq_\successor \colon (0,1)\times \Successors^* \to  2^{[0,1]}$ for successor $\successor\in \Successors$, then
$\Probability[ \forall \sampletime\in\PosNaturals \colon  \saprob( \successor) \in \confseq_\successor(\error; Z_{1:\sampletime}) ] \geq 1-\error$ if and only if $\Probability[ \saprob( \successor) \in \confseq_\successor(\error; Z_{1:\tau})  ] \geq 1-\error$ for all stopping times $\tau$.
\end{restatable}

\begin{proof}[of~\cref{lemma:cs_equivalence}]
A similar, but more condensed proof can be found in~\cite{howard2021time}. 
    Let $C_\sampletime$ be $\mathcal{F}_\sampletime$-measurable sets and fix $\mu$.

\emph{($\Rightarrow$)}
Since $\{\mu\notin C_\tau\}\subseteq \{\exists \sampletime\in\PosNaturals:\mu\notin C_\sampletime\}$, we have
\begin{align*}
\Probability[\mu\notin C_\tau]\leq \Probability[\exists \sampletime \colon \mu\notin C_\sampletime]\leq \error,
\end{align*}
so $\Probability[\mu\in C_\tau]\geq 1-\error$.

\emph{($\Leftarrow$)}
Define $\tau\coloneqq \inf\{\sampletime\in\PosNaturals:\mu\notin C_\sampletime\}$ and $\tau_n\coloneqq \min(\tau,n)$.
Then $\tau_n$ is a stopping time and $\{\tau\leq n\}\subseteq \{\mu\notin C_{\tau_n}\}$, so by the assumed stopping-time guarantee,
\begin{align*}
\Probability[\tau\leq n]\leq \Probability[\mu\notin C_{\tau_n}]\leq \error\qquad\text{for all }n.
\end{align*}
Letting $n\to\infty$ results in $\Probability[\tau<\infty]\leq \error$, i.e.
\begin{align*}
\Probability\left(\exists \sampletime\in\PosNaturals:\mu\notin C_\sampletime\right)\leq \error
\quad\Longleftrightarrow\quad
\Probability\left(\forall \sampletime\in\PosNaturals:\mu\in C_\sampletime\right)\geq 1-\error.
\end{align*}
\end{proof}

\csFamilyThm*
\begin{proof}[of~\cref{thrm:sound_var_cs}]
    This is a direct consequence of \cite[Thm.~1]{howard2021time} for the stitched bounds and \cite[Prop.~1 and Thm.~2]{waudby2024estimating} for the betting bounds.
\end{proof}

\subsection{Test Confidence Sequences}

\paragraph{Consistent.}
Let $\nu_\successor \colon \States^* \to \Distributions([0,1])$ map observations to a parameter distribution, define for $q\in (0,1)$.
We call the function $\nu_\successor$ consistent if $\nu_\successor(Z_{1:\sampletime})$ converges almost surely to the true parameter $\saprob(\successor)$.
Similarly, let $\nu \colon \States^* \to \Distributions(\Distributions(\Successors))$ map state observations to a distribution over successor distributions. We call the function $\nu$ consistent if $\nu(Z_{1:\sampletime})$ converges almost surely to the true successor distribution $\saprob$.

\paragraph{Notation.}
For $\successor \in \States$ and $z_{1:\sampletime}\in \States^\sampletime$ we can rewrite the condition $M_\successor^{\nu_\successor}(q; z_{1:\sampletime}) \leq 1/\error$ as
\begin{align*}
    &\underbrace{-\successes_\successor(z_{1:\sampletime}) \log(q) - (\sampletime - \successes_\successor(z_{1:\sampletime})) \cdot \log(1-q)}_{=: \alpha_\successor(q;z_{1:\sampletime}) \text{, variable in $q$} } \\
    &\leq
 \underbrace{\log(1/\error)  - {\sum}_{i=1}^\sampletime \log \int_{0}^1 u^{\indicator{z_i=\successor}} (1-u)^{1-\indicator{z_i=\successor}} \, \nu_\successor(z_{1:i-1})(du)}_{=: \beta_\successor^{\nu_\successor}(z_{1:\sampletime}) \text{, constant in $q$}} .
\end{align*}
And the condition for $\confset^{M^\nu}(\error; z_{1:\sampletime})$ as
\begin{align*}
    \underbrace{-\sum_{\successor \in \Successors} \successes_\successor(z_{1:\sampletime})\,\log Q(\successor)}_{=: \alpha(Q;z_{1:\sampletime}) \text{, variable in $Q$}} 
    \leq
    \underbrace{\log(1/\error)- \sum_{i=1}^\sampletime \log \int_{\Distributions(\Successors)} U(z_i)\,\nu(z_{1:i-1})(dU)}_{=: \beta^\nu(z_{1:\sampletime}) \text{, constant in $Q$}}.
\end{align*}
We remark that  
$\alpha(Q;z_{1:\sampletime})\leq \beta^\nu(z_{1:\sampletime}) $ is simply a rewritten form of the expression $M^\nu(Q;z_{1:\sampletime})\leq 1/\error$. Since the former separates the parts that depend on $Q$ from the parts that are constant $Q$, it is more convenient for the proofs.

\lrSingleThm*
\begin{proof}[of~\cref{thrm:sound_single_lr}]
    Fix $\error\in(0,1)$. 
    
\noindent    
For $q\in \{0,1\}$. Then $M_\sampletime(q) \coloneqq M_\successor^{\nu_\successor}(q;Z_{1:\sampletime})=1$ almost surely and therefore a test-martingale. Hence,
$\Probability(\exists \sampletime\in \PosNaturals. \,M_\successor^{\nu_\successor}(q;Z_{1:\sampletime})\geq 1/\error )=0$, which means the true parameter will never be rejected incorrectly.

For $q\in(0,1)$ and $\sampletime\in\PosNaturals$, write
\begin{align*}
M_\sampletime(q)\ \coloneqq\ M_\successor^{\nu_\successor}(q;Z_{1:\sampletime})
=\prod_{i=1}^\sampletime A_i(q),
\quad
A_i(q)\ \coloneqq\ \int_{0}^1 \frac{u^{X_i}(1-u)^{1-X_i}}{q^{X_i}(1-q)^{1-X_i}}\nu_\successor(Z_{1:i-1})(du),
\end{align*}
Clearly $M_0(q)=1$ and $M_\sampletime(q)\geq0$.

\noindent
\emph{Test-martingale property under $q$.}
Assume the null hypothesis $p=q$, i.e., $X_i\sim \mathrm{Bernoulli}(q)$ i.i.d.\ and $X_i$ is independent of $\mathcal{F}_{i-1}$.
Since $\nu_\successor(Z_{1:i-1})$ is $\mathcal{F}_{i-1}$-measurable, we can condition on $\mathcal{F}_{i-1}$ and use Fubini and the tower property
\begin{align*}
\expect_q\left[A_i(q)\mid \mathcal{F}_{i-1}\right]
&=\int_{0}^1 \expect_q\left[\frac{u^{X_i}(1-u)^{1-X_i}}{q^{X_i}(1-q)^{1-X_i}}\ \middle|\ \mathcal{F}_{i-1}\right]\nu_\successor(Z_{1:i-1})(du)\\
&=\int_{0}^1 \left(q\cdot \frac{u}{q} + (1-q)\cdot \frac{1-u}{1-q}\right)\nu_\successor(Z_{1:i-1})(du)
\\
&=\in\sampletime_0^1 1\cdot \nu_\successor(Z_{1:i-1})(du)=1.
\end{align*}
Hence, for all $\sampletime\geq1$,
\begin{align*}
\expect_q\left[M_\sampletime(q)\mid \mathcal{F}_{\sampletime-1}\right]
&=\expect_q\left[M_{\sampletime-1}(q)A_\sampletime(q)\mid \mathcal{F}_{\sampletime-1}\right]
\\
&=M_{\sampletime-1}(q)\expect_q\left[A_\sampletime(q)\mid \mathcal{F}_{\sampletime-1}\right]
=M_{\sampletime-1}(q),
\end{align*}
so $(M_\sampletime(q))_{\sampletime\in\PosNaturals}$ is a test martingale under $p=q$.

\noindent
\emph{Inversion and Ville.}
Let $\confset_\successor^{M_\successor^{\nu_\successor}}(\error;Z_{1:\sampletime})\coloneqq \{q\in[0,1]: M_\sampletime(q)\leq1/\error\}$ which clearly satisfies
\begin{align*}
   \{q\in[0,1]: M_\sampletime(q) < 1/\error\} \subseteq \{q\in[0,1]: M_\sampletime(q)\leq1/\error\} .
\end{align*}
By Ville’s inequality under $p=q$,
\begin{align*}
\Probability_q\left[\sup_{\sampletime\in\PosNaturals} M_\sampletime(q)\ \geq \tfrac{1}{\error}\right] \leq \error.
\end{align*}
For any stopping time $\tau$,
\begin{align*}
\{q\notin \confset_\successor^{M_\successor^{\nu_\successor}}(\error;Z_{1:\tau})\}
=\{M_\tau(q)> 1/\error\}
\subseteq \left\{\sup_{\sampletime\in\PosNaturals} M_\sampletime(q)\geq1/\error\right\},
\end{align*}
hence $\Probability_q\left[q\notin \confset_\successor^{M_\successor^{\nu_\successor}}(\error;Z_{1:\tau})\right]\leq\error$.
Applying this with $q=p$ provides us with
\begin{align*}
\Probability\left(p\in \confset_\successor^{M_\successor^{\nu_\successor}}(\error;Z_{1:\tau})\right)\ \geq 1-\error.
\end{align*}
The time-uniform guarantee follows directly from~\Cref{lemma:cs_equivalence}.
If $\nu_\successor$ is consistent, then the convergence statement follows by invoking~\Cref{lemma:convergence_multinomial} for a 2-dimensional multinomial.
\end{proof}

\lrMultiThm*
\begin{proof}[of~\cref{thrm:sound_multi_lr}]
    Fix $\error\in(0,1)$ and write $P\coloneqq \saprob()\in\Distributions(\Successors)$. For $Q\in\Distributions(\Successors)$ and $\sampletime\in\PosNaturals$, write
\begin{align*}
M_\sampletime(Q)\ \coloneqq\ M^{\nu}(Q;Z_{1:\sampletime})
=\prod_{i=1}^\sampletime B_i(Q),
\quad
B_i(Q)\ \coloneqq\ \int_{\Distributions(\Successors)} \frac{U(Z_i)}{Q(Z_i)}\nu(Z_{1:i-1})(dU).
\end{align*}
Clearly $M_0(Q)=1$ and $M_\sampletime(Q)\geq0$.

\noindent
\emph{Test-martingale property under $Q$.}
Assume the null hypothesis $P=Q$, i.e., $Z_i\sim Q$ i.i.d.\ and independent of $\mathcal{F}_{i-1}=\sigma(Z_{1:i-1})$.
Since $\nu(Z_{1:i-1})$ is $\mathcal{F}_{i-1}$-measurable, we have
\begin{align*}
\expect_Q\left[B_i(Q)\mid \mathcal{F}_{i-1}\right]
&=\int_{\Distributions(\Successors)} \expect_Q\left[\frac{U(Z_i)}{Q(Z_i)}\ \middle|\ \mathcal{F}_{i-1}\right]\nu(Z_{1:i-1})(dU)\\
&=\int_{\Distributions(\Successors)} \left(\sum_{\successor\in\Successors} Q(\successor)\frac{U(\successor)}{Q(\successor)}\right)\nu(Z_{1:i-1})(dU)\\
&=\int_{\Distributions(\Successors)} \left(\sum_{\successor\in\Successors}U(\successor) \right)\nu(Z_{1:i-1})(dU)\\
&=\int_{\Distributions(\Successors)} 1\cdot \nu(Z_{1:i-1})(dU)=1.
\end{align*}
Therefore,
\begin{align*}
\expect_Q\left[M_\sampletime(Q)\mid \mathcal{F}_{\sampletime-1}\right]
=\expect_Q\left[M_{\sampletime-1}(Q)B_\sampletime(Q)\mid \mathcal{F}_{\sampletime-1}\right]
=M_{\sampletime-1}(Q),
\end{align*}
so $(M_\sampletime(Q))_{\sampletime\in\PosNaturals}$ is a test martingale under $P=Q$.

\noindent
\emph{Inversion and Ville.}
Let $\confset^{M^\nu}(\error;Z_{1:\sampletime})\coloneqq \{Q\in\Distributions(\Successors): M_\sampletime(Q)\leq 1/\error\}$. Ville’s inequality under $P=Q$ provides the guarantee that
\begin{align*}
\Probability_Q\left[\sup_{\sampletime\in\PosNaturals} M_\sampletime(Q)\ \geq \tfrac{1}{\error}\right]\ \leq  \error.
\end{align*}
For any stopping time $\tau$,
\begin{align*}
\{Q\notin \confset^{M^\nu}(\error;Z_{1:\tau})\}
=\{M_\tau(Q) > 1/\error\}
\subseteq \left\{\sup_{\sampletime\in\PosNaturals} M_\sampletime(Q)\geq1/\error\right\}.
\end{align*}
Therefore, $\Probability_Q\left[Q\notin \confset^{M^\nu}(\error;Z_{1:\tau})\right]\leq\error$.
Applying this with $Q=P$ gives
\begin{align*}
\Probability\left(P\in \confset^{M^\nu}(\error;Z_{1:\tau})\right)\ \geq 1-\error.
\end{align*}
The time-uniform guarantee follows directly from~\Cref{lemma:cs_equivalence}. If $\nu$ is consistent, then the convergence statement is exactly~\Cref{lemma:convergence_multinomial}.
\end{proof}

\alphaConvexLemma*
\begin{proof}[of~\cref{lemma:convex}]
Fix $\error\in(0,1)$, $\successor\in\Successors$, and $z_{1:\sampletime}\in\Successors^\sampletime$. Recall
\begin{align*}
\alpha_\successor(q;z_{1:\sampletime})
&\coloneqq
-\successes_\successor(z_{1:\sampletime})\log q - \left(t-\successes_\successor(z_{1:\sampletime})\right)\log(1-q),
\qquad q\in(0,1).
\end{align*}
Let $k\coloneqq \successes_\successor(z_{1:\sampletime})$.
For $q\in(0,1)$,
\begin{align*}
\alpha_\successor'(q;z_{1:\sampletime})
&=
-\frac{k}{q}
+\frac{t-k}{1-q},
\quad \text{and} \quad 
\alpha_\successor''(q;z_{1:\sampletime})
=
\frac{k}{q^2}
+\frac{t-k}{(1-q)^2}.
\end{align*}
Since $k\in\{0,1,\dots,t\}$ and $q\in(0,1)$, we have
\begin{align*}
\alpha_\successor''(q;z_{1:\sampletime}) \geq 0 \qquad \text{for all } q\in(0,1),
\end{align*}
and hence $\alpha_\successor(\cdot;z_{1:\sampletime})$ is convex on $(0,1)$.

If $0<k<t$, then $\alpha_\successor''(q;z_{1:\sampletime})>0$ for all $q\in(0,1)$, so $\alpha_\successor(\cdot;z_{1:\sampletime})$ is strictly convex.
Moreover,
\begin{align*}
\alpha_\successor'(q;z_{1:\sampletime})=0
&\iff
-\frac{k}{q}+\frac{t-k}{1-q}=0
\iff
q=\frac{k}{\sampletime}=\average_\successor(z_{1:\sampletime}).
\end{align*}
Thus, if $0<\average_\successor(z_{1:\sampletime})<1$, the unique minimiser is $q=\average_\successor(z_{1:\sampletime})$.
If $k=0$, then
\begin{align*}
\alpha_\successor'(q;z_{1:\sampletime})=\frac{\sampletime}{1-q}>0
\qquad\text{for all }q\in(0,1),
\end{align*}
so $\alpha_\successor(\cdot;z_{1:\sampletime})$ is strictly increasing on $(0,1)$.
If $k=t$, then
\begin{align*}
\alpha_\successor'(q;z_{1:\sampletime})=-\frac{\sampletime}{q}<0
\qquad\text{for all }q\in(0,1),
\end{align*}
so $\alpha_\successor(\cdot;z_{1:\sampletime})$ is strictly decreasing on $(0,1)$.
In all three cases, every sublevel set of $\alpha_\successor(\cdot;z_{1:\sampletime})$ in $(0,1)$ is an interval.
\end{proof}

\kktProjectionThm*

\begin{proof}[of~\cref{thrm:kkt_projection}]
Fix $\sampletime\in\PosNaturals$, $\error\in(0,1)$, $z_{1:\sampletime}\in\States^\sampletime$, and $V\colon\Successors\to\Reals$.
Write $N_\sampletime(\successor)\coloneqq \successes_\successor(z_{1:\sampletime})$,
$\alpha_\sampletime(Q)\coloneqq \alpha(Q;z_{1:\sampletime})$, and $\beta_\sampletime\coloneqq \beta^\nu(z_{1:\sampletime})$.

\noindent
First we deal with the zero-count states. Let $ \Successors_\sampletime\coloneqq \{\successor\in\Successors\mid N_\sampletime(\successor)>0\}.$
If $\Successors_\sampletime\neq \Successors$, then the constraint
\begin{align*}
\alpha_\sampletime(Q)
=
-\sum_{\successor\in\Successors} N_\sampletime(\successor)\log Q(\successor)
=
-\sum_{\successor\in\Successors_\sampletime} N_\sampletime(\successor)\log Q(\successor)
\end{align*}
depends only on $(Q(\successor))_{\successor\in\Successors_\sampletime}$.
Hence we may first solve the optimisation over $\Successors_\sampletime$ and then allocate any remaining mass
\begin{align*}
    m\coloneqq 1-\sum_{\successor\in\Successors_\sampletime}Q(\successor)
\end{align*}
to an extremiser of $V$ over $\Successors\setminus \Successors_\sampletime$
(argmin for the infimum, argmax for the supremum).
Thus, it suffices to treat the case $\Successors_\sampletime=\Successors$,
equivalently $N_\sampletime(\successor)>0$ for all $\successor\in\Successors$. 

\noindent
We restate the minimisation problem as
\begin{align*}
\min_{Q\in\Reals^{\Successors}}\quad &\sum_{\successor\in\Successors} Q(\successor)V(\successor)\\
\text{s.t.}\quad
&\alpha_\sampletime(Q)\leq \beta_\sampletime,\qquad
\sum_{\successor\in\Successors}Q(\successor)=1,\qquad
Q(\successor)>0\ \ \forall \successor\in\Successors.
\end{align*}
Let $\hat{P}_\sampletime \coloneqq N_\sampletime(\cdot)/\sampletime$, the empirical distribution.
We distinguish two cases.

\noindent
\emph{Case 1:} $\alpha_\sampletime(\hat{P}_\sampletime)=\beta_\sampletime$.
Since $\hat{P}_\sampletime$ minimises $\alpha_\sampletime$ over the simplex and is feasible, every feasible
$Q$ satisfies $\beta_\sampletime \geq \alpha_\sampletime(Q) \ge \alpha_\sampletime(\hat{P}_\sampletime)=\beta_\sampletime$.
Hence $\alpha_\sampletime(Q)=\beta_\sampletime$ for every feasible $Q$.
By strict convexity of $\alpha_\sampletime$, the minimiser over the simplex is unique, so
$Q=\hat{P}_\sampletime$. Hence, the releasable set is $\{\hat{P}_\sampletime\}$ and $\hat{P}_\sampletime$ is the minimiser.

\noindent
\emph{Case 2:} $\alpha_\sampletime(\hat{P}_\sampletime)<\beta_\sampletime$.
Then $\hat{P}_\sampletime$ is a strictly feasible point, so Slater's condition holds.
Hence, the KKT conditions are necessary and sufficient.
We introduce Lagrange multipliers $\eta\in\Reals$ for the simplex constraint and $\lambda\geq 0$ for the log-constraint, and define
\begin{align*}
\mathcal L(Q,\eta,\lambda)
\coloneqq
\sum_{\successor\in\Successors}Q(\successor)V(\successor)
+\eta\left(\sum_{\successor\in\Successors}Q(\successor)-1\right)
+\lambda\left(\alpha_\sampletime(Q)-\beta_\sampletime\right).
\end{align*}
Due to stationarity with respect to $Q(\successor)$, we know that the gradient of the Lagrangian w.r.t. the coordinates of $Q$ must vanish in the optimum. Hence, we take the derivative and set it to $0$, i.e., for all $\successor\in\Successors$
\begin{align*}
\frac{\partial \mathcal L}{\partial Q(\successor)}
=
V(\successor)+\eta-\lambda\frac{N_\sampletime(\successor)}{Q(\successor)}
=
0
\quad\text{and therefore} \quad Q(\successor)=\frac{\lambda N_\sampletime(\successor)}{\eta+V(\successor)}.
\end{align*}
Thus necessarily $\eta>-\min_{\successor\in\Successors}V(\successor)$, as otherwise the expression becomes negative or undefined.
Using $\sum_\successor Q(\successor)=1$ we obtain
\begin{align*}
1
=
\sum_{\successor\in\Successors}\frac{\lambda N_\sampletime(\successor)}{\eta+V(\successor)}
\qquad\Longleftrightarrow\qquad
\lambda
=
\left(\sum_{u\in\Successors}\frac{N_\sampletime(u)}{\eta+V(u)}\right)^{-1},
\end{align*}
hence every KKT point is of the form
\begin{align*}
Q_\eta^\star(\successor)
=
\frac{\frac{N_\sampletime(\successor)}{\eta+V(\successor)}}{\sum_{u\in\Successors}\frac{N_\sampletime(u)}{\eta+V(u)}}
\qquad\text{for all }\successor\in\Successors.
\end{align*}
If $\lambda=0$, then stationarity implies $V(\successor)+\eta=0$ which means that $V$ is constant on $\Successors$, contrary to assumption.
Hence $\lambda>0$, and complementary slackness implies that the log-constraint is active:
\begin{align*}
\alpha_\sampletime(Q_\eta^\star)=\beta_\sampletime.
\end{align*}
It remains to prove existence and uniqueness of $\eta$.
A direct computation gives
\begin{align*}
\alpha_\sampletime(Q_\eta^\star)
=
-\sum_{\successor\in\Successors}N_\sampletime(\successor)\log N_\sampletime(\successor)
+\sum_{\successor\in\Successors}N_\sampletime(\successor)\log(\eta+V(\successor))
+\sampletime\log\left(\sum_{u\in\Successors}\frac{N_\sampletime(u)}{\eta+V(u)}\right).
\end{align*}
Hence $\alpha_\sampletime(Q_\eta^\star)$ is continuous in $\eta$ on $(-\min_u V(u),\infty)$.
We show that $\alpha_\sampletime(Q_\eta^\star)$ is strictly decreasing in $\eta$ by taking the derivative, i.e.,
\begin{align*}
\frac{d}{d\eta}\alpha_\sampletime(Q_\eta^\star)
&=
A(\eta)-\sampletime\frac{B(\eta)}{A(\eta)}
=
\frac{A(\eta)^2-\sampletime B(\eta)}{A(\eta)}.
\end{align*}
where
\begin{align*}
A(\eta)\coloneqq \sum_{u\in\Successors}\frac{N_\sampletime(u)}{\eta+V(u)},
\qquad
B(\eta)\coloneqq \sum_{u\in\Successors}\frac{N_\sampletime(u)}{(\eta+V(u))^2}.
\end{align*}
By Cauchy-Schwarz,
\begin{align*}
A(\eta)^2
=
\left(\sum_{u\in\Successors}\frac{\sqrt{N_\sampletime(u)}\sqrt{N_\sampletime(u)}}{\eta+V(u)}\right)^2
<
\left(\sum_{u\in\Successors}\frac{N_\sampletime(u)}{(\eta+V(u))^2}\right)
\left(\sum_{u\in\Successors}N_\sampletime(u)\right)
=
\sampletime B(\eta),
\end{align*}
where the inequality is strict because $V$ is not constant on $\Successors$.
Thus
\begin{align*}
\frac{d}{d\eta}\alpha_\sampletime(Q_\eta^\star)<0
\qquad\text{for all }\eta>-\min_{u\in\Successors}V(u),
\end{align*}
Finally, if we take $\eta\downarrow -\min_{u\in\Successors}V(u)$, at least one denominator tends to $0$ implying that $\alpha_\sampletime(Q_\eta^\star)\to +\infty$, and if we take $\eta\to\infty$, we get $Q_\eta^\star(\successor)\to N_\sampletime(\successor)/\sampletime$, together they imply $ \alpha_\sampletime(Q_\eta^\star)\to \alpha_\sampletime\left(N_\sampletime(\cdot)/\sampletime\right)$.
Therefore, by continuity and strict monotonicity, there exists a unique $\eta>-\min_u V(u)$ such that $\alpha_\sampletime(Q_\eta^\star)=\beta_\sampletime$. This yields the claimed minimiser.
The supremum problem is identical after replacing $V$ by $-V$.
\end{proof}

\begin{restatable}{lemma}{lrMultiConv}
    \label{lemma:convergence_multinomial}
Let $\error \in (0,1)$ and $M^\nu$ as in~\cref{eq:LRset} with a consistent $\nu$ and let $C_\sampletime=\confset^{M^\nu}(\error;Z_{1:\sampletime}).$
Then, for every $\varepsilon > 0$, there exists an almost surely finite random time $T_\varepsilon$ such that $ C_\sampletime  \subset \{Q \in \Distributions(\Successors) \mid \|Q-\saprob\|_1 < \varepsilon\}$ for all $ \sampletime \geq T_\varepsilon$
almost surely.
Moreover, on the event $\{\sup_{\sampletime \in\PosNaturals} M^\nu(\saprob;Z_{1:\sampletime}) \leq 1/\error\}$, we have
$\saprob\in C_\sampletime$ for all $ \sampletime \in\PosNaturals$ and  $\mathrm{diam}_{\|\cdot \|_1}(C_\sampletime) \to 0$ as $\sampletime \to \infty$.
\end{restatable}

\begin{proof}[of~\cref{lemma:convergence_multinomial}]
For $i\in\PosNaturals$ define the predictive mean
\begin{align*}
L_{i-1}(\successor)\ \coloneqq\ \int_{\Distributions(\Successors)} U(\successor)\nu(Z_{1:i-1})(dU)\in[0,1],\qquad \successor\in\Successors,
\end{align*}
so that
\begin{align*}
M^\nu(Q;Z_{1:\sampletime})=\prod_{i=1}^{\sampletime} \frac{L_{i-1}(Z_i)}{Q(Z_i)}.
\end{align*}
Let $P \in \Distributions(\Successors)$.
First, we show the growth rate of the test martingale.
Fix any $Q\in\Distributions(\Successors)$. If $Q(\successor)=0$ for some $\successor\in\Successors$, then $Z_i=\successor$ occurs infinitely often a.s., hence
$M^\nu(Q;Z_{1:\sampletime})=+\infty$ for all sufficiently large $\sampletime$.
Thus assume $Q(\successor)>0$ for all $\successor\in\Successors$.
Taking logs and adding/subtracting $\log P(Z_i)$ yields
\begin{align*}
\frac{1}{\sampletime}\log M^\nu(Q;Z_{1:\sampletime})
&=\frac{1}{\sampletime}\sum_{i=1}^{\sampletime}\left(\log L_{i-1}(Z_i)-\log P(Z_i)\right)
+\frac{1}{\sampletime}\sum_{i=1}^{\sampletime} \log\frac{P(Z_i)}{Q(Z_i)}.
\end{align*}
By consistency, for every $\successor\in\Successors$ we have $L_{i-1}(\successor)\to P(\successor)$ a.s. Since $\Successors$ is finite and $P(\successor)>0$ for all $\successor$, this implies
\begin{align*}
\max_{\successor\in\Successors}\left|\log L_{i-1}(\successor)-\log P(\successor)\right|\to 0
\qquad\text{a.s.}
\end{align*}
Hence also $\log L_{i-1}(Z_i)-\log P(Z_i)\to 0$ a.s., and therefore
\begin{align*}
\frac{1}{\sampletime}\sum_{i=1}^{\sampletime}\left(\log L_{i-1}(Z_i)-\log P(Z_i)\right)\to 0
\qquad\text{a.s.}
\end{align*}
For the second term, let $\widehat P_{\sampletime}(\successor)\coloneqq \xi_{\successor}(Z_{1:\sampletime})/\sampletime$ denote the empirical distribution. Then
\begin{align*}
\frac{1}{\sampletime}\sum_{i=1}^{\sampletime} \log\frac{P(Z_i)}{Q(Z_i)}
&=
\sum_{\successor\in\Successors}\widehat P_{\sampletime}(\successor)\log\frac{P(\successor)}{Q(\successor)}\\
&=
\sum_{\successor\in\Successors}\widehat P_{\sampletime}(\successor)\log\frac{\widehat P_{\sampletime}(\successor)}{Q(\successor)}
-
\sum_{\successor\in\Successors}\widehat P_{\sampletime}(\successor)\log\frac{\widehat P_{\sampletime}(\successor)}{P(\successor)}\\
&=
\KL(\widehat P_{\sampletime}\|Q)-\KL(\widehat P_{\sampletime}\|P),
\end{align*}
where we use the convention $0\log(0/x)=0$.
By the strong law of large numbers, $\widehat P_{\sampletime}\to P$ a.s. Since $P$ has full support, also $\KL(\widehat P_{\sampletime}\|P)\to 0$ a.s. Moreover, for every fixed $Q\neq P$, $\KL(\widehat P_{\sampletime}\|Q)\to \KL(P\|Q)$ and therefore $\frac{1}{\sampletime}\log M^\nu(Q;Z_{1:\sampletime})\to \KL(P\|Q)$, which implies $M^\nu(Q;Z_{1:\sampletime})\to\infty$ almost surely.

\noindent
Fix $\varepsilon>0$ and define
\begin{align*}
K_\varepsilon\coloneqq\{Q\in\Distributions(\Successors):\|Q-P\|_1\geq\varepsilon\}.
\end{align*}
To obtain a uniform bound over $K_\varepsilon$, let
\begin{align*}
A_{\sampletime}\coloneqq
\frac{1}{\sampletime}\sum_{i=1}^{\sampletime}\left(\log L_{i-1}(Z_i)-\log P(Z_i)\right),
\end{align*}
so that for every $Q\in\Distributions(\Successors)$,
\begin{align*}
\frac{1}{\sampletime}\log M^\nu(Q;Z_{1:\sampletime})
=
A_{\sampletime}
+\KL(\widehat P_{\sampletime}\|Q)-\KL(\widehat P_{\sampletime}\|P).
\end{align*}
By the previous part, $A_{\sampletime}\to 0$, $\KL(\widehat P_{\sampletime}\|P)\to 0$, and $\widehat P_{\sampletime}\to P$ a.s., so there exists an a.s.\ finite random time $T_\varepsilon^{(1)}$ such that for all $\sampletime\geq T_\varepsilon^{(1)}$,
\begin{align*}
\|\widehat P_{\sampletime}-P\|_1<\frac{\varepsilon}{2}.
\end{align*}
Hence, by the reverse triangle inequality for every $Q\in K_\varepsilon$ and every $\sampletime\geq T_\varepsilon^{(1)}$,
\begin{align*}
\|\widehat P_{\sampletime}-Q\|_1
\geq
\|Q-P\|_1-\|\widehat P_{\sampletime}-P\|_1
\geq
\frac{\varepsilon}{2}.
\end{align*}
Therefore, by Pinsker's inequality,
\begin{align*}
\KL(\widehat P_{\sampletime}\|Q)\geq \frac{1}{2}\left(\frac{\varepsilon}{2}\right)^2=\frac{\varepsilon^2}{8}
\qquad
\text{for all $Q\in K_\varepsilon$ and all $\sampletime\geq T_\varepsilon^{(1)}$.}
\end{align*}
Since $A_{\sampletime}\to 0$ and $\KL(\widehat P_{\sampletime}\|P)\to 0$ a.s., there exists an a.s.\ finite random time $T_\varepsilon^{(2)}$ such that for all $\sampletime\geq T_\varepsilon^{(2)}$,
\begin{align*}
|A_{\sampletime}| \leq \frac{\varepsilon^2}{32}
\qquad\text{and}\qquad
\KL(\widehat P_{\sampletime}\|P)\leq \frac{\varepsilon^2}{32}.
\end{align*}
Thus, for all $\sampletime\geq T_\varepsilon\coloneqq \max\{T_\varepsilon^{(1)},T_\varepsilon^{(2)}\}$ and all $Q\in K_\varepsilon$,
\begin{align*}
\frac{1}{\sampletime}\log M^\nu(Q;Z_{1:\sampletime})
&\geq
-\frac{\varepsilon^2}{32}
+\frac{\varepsilon^2}{8}
-\frac{\varepsilon^2}{32}
=
\frac{\varepsilon^2}{16}.
\end{align*}
Hence, for all $\sampletime\geq T_\varepsilon$ and all $Q\in K_\varepsilon$,
\begin{align*}
M^\nu(Q;Z_{1:\sampletime}) \geq \exp\left(\frac{\varepsilon^2}{16}\sampletime\right).
\end{align*}
By increasing $T_\varepsilon$ further if necessary, we obtain for all $\sampletime\geq T_\varepsilon$ and all $Q\in K_\varepsilon$,
\begin{align*}
    M^\nu(Q;Z_{1:\sampletime}) \geq \frac{1}{\error}.
\end{align*}
Therefore, for all $\sampletime\geq T_\varepsilon$ we have $K_\varepsilon\cap C_\sampletime=\emptyset$, i.e.\
$C_\sampletime\subset \{Q:\|Q-P\|_1<\varepsilon\}$.
Let 
\begin{align*}
    E\coloneqq \left\{\sup_{\sampletime \in\PosNaturals} M^\nu(P;Z_{1:\sampletime}) \leq 1/\error\right\}.
\end{align*}
By Ville’s inequality under $P$, $\Probability[E]\geq 1-\error$. We now condition on $E$. Hence, we have $M^\nu(P;Z_{1:\sampletime})<1/\error$ for all $\sampletime$, and thus $P\in C_\sampletime$ for all $\sampletime$.
Since the previous step holds for every $\varepsilon>0$, it follows that on $E$ the sets $C_\sampletime$ eventually lie in every $\ell_1$-ball around $P$,
so $\mathrm{diam}_{\|\cdot\|_1}(C_\sampletime)\to 0$.
\end{proof}

\section{Further Experimental Data} \label{app:data}

\subsection{Individual Inference} \label{app:data:individual}
In this section, we present further data on the evaluation of our inference methods on individual distributions.
As distribution classes, we consider 2, 5, 10, and 50 successors, and uniform, linear, squared and one-vs-all probability distribution, as follows:
Uniform distributions assign $1 / n$ probability to each successor, linear / squared assign a relative probability of $i$ / $i^2$ to the $i$-th successor, and one-vs-all assigns $0.5$ to the first and $0.5 / n$ to all others.
We then add some noise by scaling each probability with a random factor between $0.9$ and $1.1$.
The results are summarized in \cref{fig:app:uniform,fig:app:linear,fig:app:squared,fig:app:oneVall}.

Before further discussion of the results, we remark that our implementation of Clopper-Pearson (relying on the \texttt{apache-commons} library) runs into numerical issues for large sample counts and small confidences, due to the inverse Beta-distribution (as well as the the alternative F-distribution) lacking a closed form.
Similarly, our (naive) implementation of \texttt{ValueTest} sometimes exhibits numerical instability in the computation.
We took care in our implementation to round conservatively, ensuring correctness.
We also implemented a variant using arbitrary precision rationals, however this seems too slow for quick evaluation.

We complement the additional data with some further discussion:
Firstly, we comment on the discussion of asymptotical convergence in \cref{subsec:union_bound}.
For the exponentially spaced methods (suffix \texttt{-Exp}), we clearly see the gaps of increasing length (note the fact that the x-axis is log-scale), with the estimates stagnating. 
The restrained spending of exponentially spaced grids often yields quite good estimates at the individual grid points, as it \enquote{concentrates} a lot of confidence on these points. 
Nonetheless, they are (eventually) outperformed by several confidence sequences such as \texttt{Bern-Bet}, \texttt{Bern-Stitch}, and \texttt{ValueTest}.
When comparing polynomially and exponential gridding, recall that \cref{lemma:poly_dominance} only considers asymptotic worst case, while we investigate a finite horizon on a fixed set of distributions.
Hence, the exponential methods outperforming the sequential ones is not a contradiction to this lemma.
Indeed, since the polynomial grid is much denser (gap-less), a significant penality is to be expected initially.

Secondly, we compare the different confidence sequences:
As for fixed-time confidence intervals (see~\cite{WatO}), the sequences based on Hoeffding's inequality are overly conservative. 
Among the others, the differences are less pronounced, but still \texttt{Bern-Bet}, \texttt{Bern-Stitch}, and \texttt{ValueTest} perform well across all shapes and successor counts.
In particular, in almost all plots, our novel \texttt{ValueTest} obtains the best precision for $10^5$ samples, often dominating before that already.

\newcommand{\fullplot}[1]{
    \singleplot{#1}{Hoeffding Union-Exp}{Hoeff-Exp}
    \singleplot{#1}{Bernstein Union-Exp}{Bern-Exp}
    \singleplot{#1}{CP Union-Exp}{CP-Exp}
    \singleplot{#1}{Hoeffding Union-Sq}{Hoeff-Sq}
    \singleplot{#1}{Bernstein Union-Sq}{Bern-Sq}
    \singleplot{#1}{CP Union-Sq}{CP-Sq}
    \singleplot{#1}{Hoeffding Stitched}{Hoeff-Stitch}
    \singleplot{#1}{Bernstein Stitched}{Bern-Stitch}
    \singleplot{#1}{Hoeffding Betting}{Hoeff-Bet}
    \singleplot{#1}{Bernstein Betting}{Bern-Bet}
    \singleplot{#1}{Maximum Likelihood Test}{MLETest}
    \singleplot{#1}{Beta Test}{BetaTest}
    \singleplot{#1}{Value Test}{ValueTest}
    \singleplot[dash pattern=on 0.2pt off .8pt]{#1}{Point CP}{Point CP}
}

\newcommand{\legendplot}[2][]{
    \addplot+[#1] [draw=none, mark=none] coordinates {(0,0)};
    \addlegendentry{\texttt{#2}}
}

\newcommand{\fulllegend}{
    \begin{tikzpicture}
    \begin{axis}[
            legend style={
                draw=none,
                row sep=2pt,
                inner sep=0pt
            },
            hide axis,
            legend columns=4,
            legend image post style={line width=1pt},
            EvalList
    ]
    \legendplot{Hoeff-Exp}
    \legendplot{Bern-Exp}
    \legendplot{CP-Exp}
    \legendplot{Hoeff-Sq}
    \legendplot{Bern-Sq}
    \legendplot{CP-Sq}
    \legendplot{Hoeff-Stitch}
    \legendplot{Bern-Stitch}
    \legendplot{Hoeff-Bet}
    \legendplot{Bern-Bet}
    \legendplot{MLETest}
    \legendplot{BetaTest}
    \legendplot{ValueTest}
    \legendplot[dash pattern=on 0.2pt off .8pt]{Point CP}
    \end{axis}
    \end{tikzpicture}
}
\newcommand{\fullplotpicture}[2]{
    \begin{tikzpicture}
    \begin{axis}[
            width=.5\textwidth,
    		height=.4\textheight,
            every axis legend/.code={\let\addlegendentry\relax},
            sampleplot,
		      xlabel=#1,
            xmode=log,ymode=log,
            yticklabels={},
            EvalList
    	]
        \fullplot{#2}
    \end{axis}
    \end{tikzpicture}
}

\begin{figure}[p]
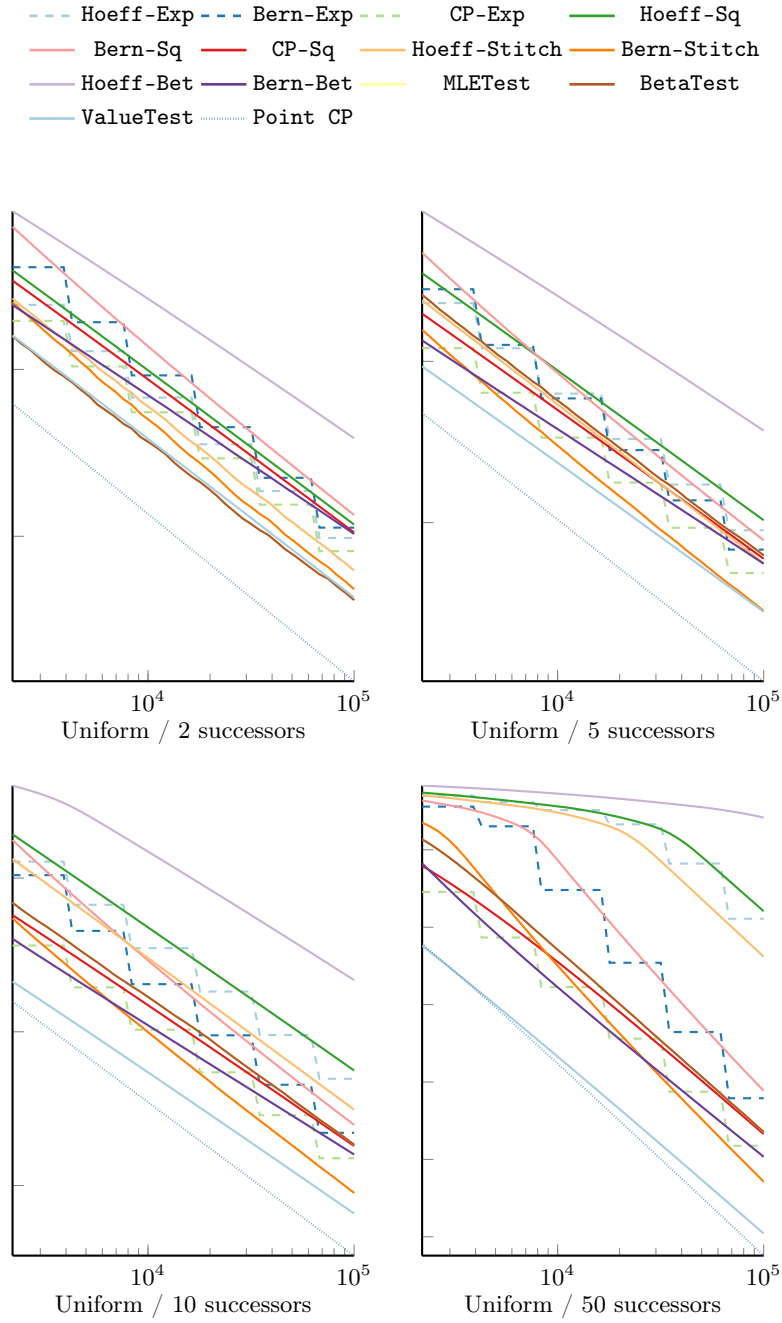

    \centering
    \fulllegend
    \fullplotpicture{Uniform / 2 successors}{Uniform_2_0.01}
    \fullplotpicture{Uniform / 5 successors}{Uniform_5_0.01}

    \vspace{.5cm}
    
    \fullplotpicture{Uniform / 10 successors}{Uniform_10_0.01}
    \fullplotpicture{Uniform / 50 successors}{Uniform_50_0.01}
    \caption{
        Bounds on uniform distributions.
    }
    \label{fig:app:uniform}
\end{figure}

\begin{figure}[p]
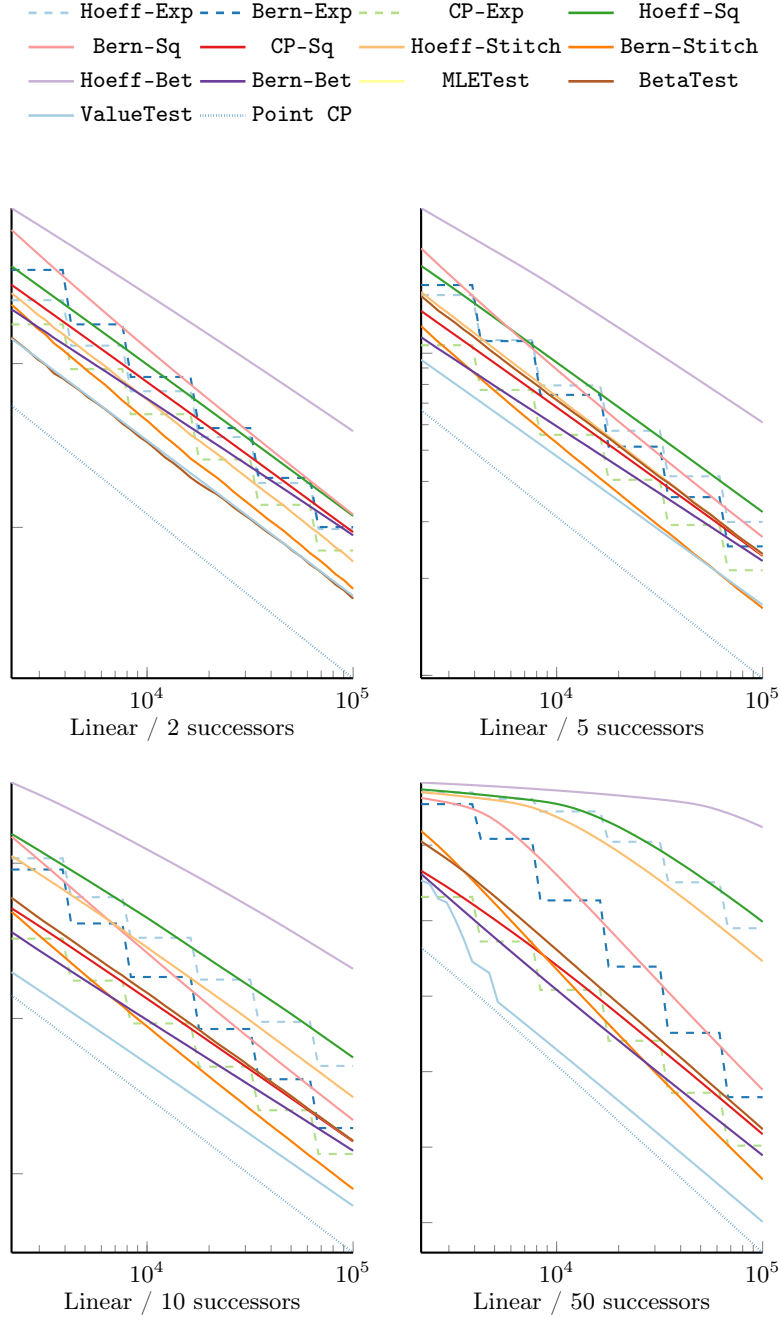

    \centering
    \fulllegend
    \fullplotpicture{Linear / 2 successors}{Linear_2_0.01}
    \fullplotpicture{Linear / 5 successors}{Linear_5_0.01}

    \vspace{.5cm}
    
    \fullplotpicture{Linear / 10 successors}{Linear_10_0.01}
    \fullplotpicture{Linear / 50 successors}{Linear_50_0.01}
    \caption{
        Bounds on linear distributions.
    }
    \label{fig:app:linear}
\end{figure}

\begin{figure}[p]
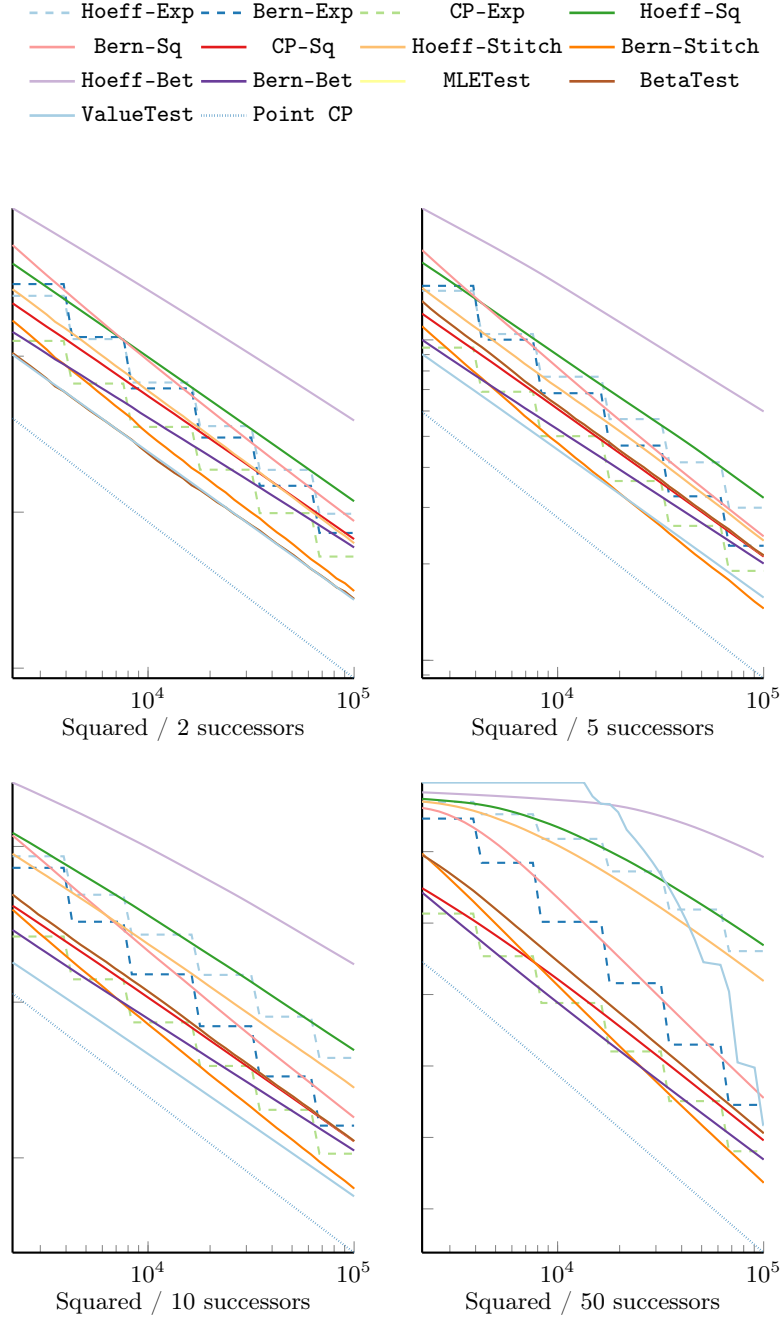

    \centering
    \fulllegend
    \fullplotpicture{Squared / 2 successors}{Squared_2_0.01}
    \fullplotpicture{Squared / 5 successors}{Squared_5_0.01}

    \vspace{.5cm}
    
    \fullplotpicture{Squared / 10 successors}{Squared_10_0.01}
    \fullplotpicture{Squared / 50 successors}{Squared_50_0.01}
    \caption{
        Bounds on squared distributions.
    }
    \label{fig:app:squared}
\end{figure}

\begin{figure}[p]
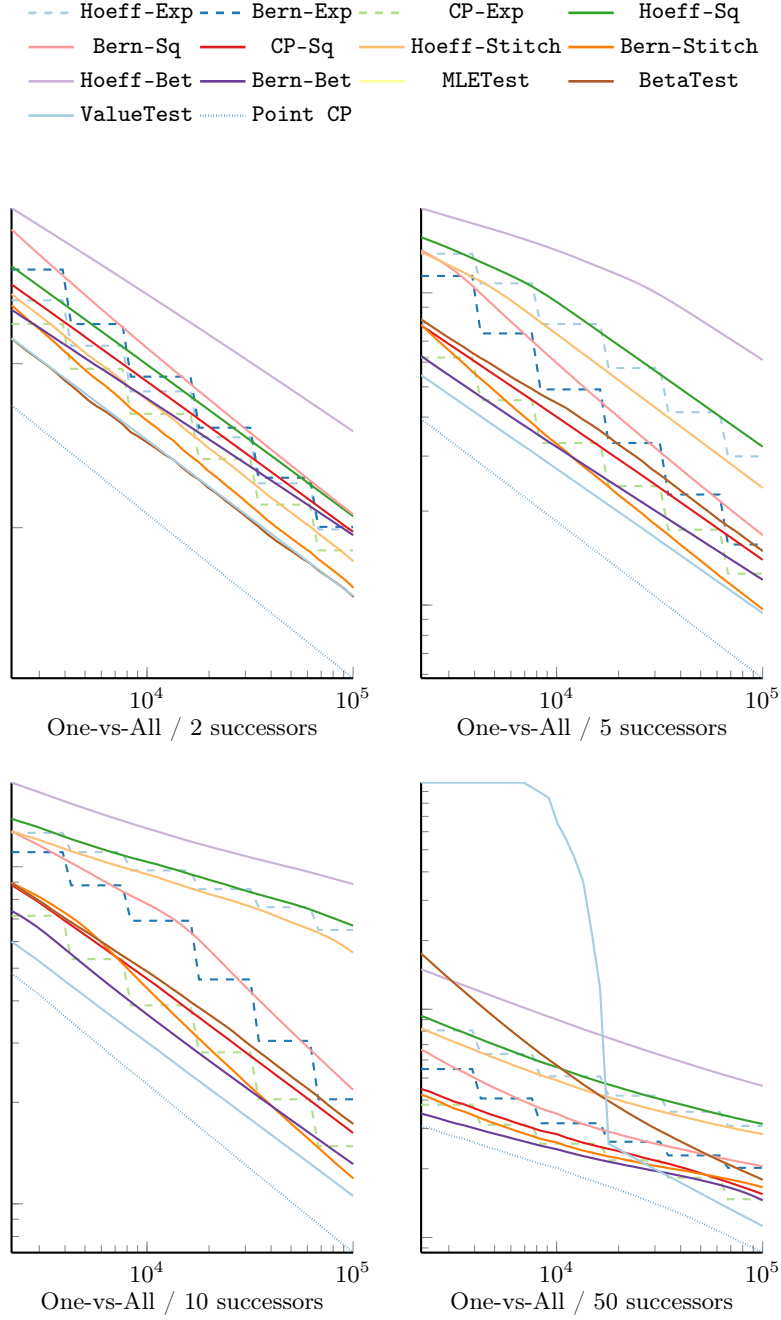

    \centering
    \fulllegend
    \fullplotpicture{One-vs-All / 2 successors}{1-vs-all_2_0.01}
    \fullplotpicture{One-vs-All / 5 successors}{1-vs-all_5_0.01}

    \vspace{.5cm}
    
    \fullplotpicture{One-vs-All / 10 successors}{1-vs-all_10_0.01}
    \fullplotpicture{One-vs-All / 50 successors}{1-vs-all_50_0.01}
    \caption{
        Bounds on one-vs-all distributions.
    }
    \label{fig:app:oneVall}
\end{figure}
\clearpage

\subsection{Online MDP-SMC} \label{app:data:mdp-smc}
In this section, we provide further data on our MDP-SMC evaluation.
\cref{tab:mdp_smc_models} lists all models and properties together with basic data about them.
\cref{tab:app:mdp_smc_data} then provides the averaged sample counts for every benchmark and method we consider, i.e.\ the full data used to generate the aggregated scores in \cref{tab:mdp_smc_data}.

\begin{table}[t]
    \centering
    \sisetup{
      detect-all
    }
    \setlength{\tabcolsep}{3pt}
    \caption{
        Description of all model / property combinations we consider, together with basic information about the model (number of states $|S|$, number of actions $|A|$, and average number of successors in each action Avg $|\Delta|$).
        Finally, we list the (arbitrarily chosen) precision and confidence values.
    }
    \scriptsize
    \begin{tabular}{>{\ttfamily}c>{\ttfamily}c>{\ttfamily}cS[table-format=6.0]S[table-format=6.0]S[table-format=1.2,round-mode=places,round-precision=2]S[table-format=1.2]S[table-format=1.2]}
        \normalfont Model & \normalfont Constants & \normalfont Property & $|S|$ & $|A|$ & Avg $|\Delta|$ & $\varepsilon$ & $\gamma$ \\
        \midrule
        bigmec & N=1000 & prob\_max  & 2003 & 4004 & 1.00 & 0.05 & 0.99 \\
        consensus-mec & N=2,K=2  & disagree & 2720 & 6680 & 1.4988 & 0.3 & 0.9 \\
        consensus & N=2,K=1 & c2 & 144 & 208 & 1.2115 & 0.2 & 0.9 \\
        csma & N=2,K=2 & all\_before\_max & 1038 & 1053 & 1.2165 & 0.1 & 0.9 \\
        evade-mdp & N=3,slippery=0.7 & prob\_max & 613 & 1545 & 2.92427 & 0.2 & 0.9 \\
        firewire\_dl & delay=3,deadline=200 & deadline & 14824 & 16671 & 1.0561 & 0.05 & 0.9 \\
        manymecs & N=20 & prob\_max & 62 & 122 & 1.49180 & 0.05 & 0.9 \\
        mer & n=1,x=0.1 & prob\_max & 7526 & 17773 & 1.0185 & 0.05 & 0.9 \\
        pacman & MAXSTEPS=7 & crash & 1272 & 1444 & 1.0311634 & 0.05 & 0.9 \\
        random-grid & N=10,pstay=0.3,pover=0.3 & prob\_max & 1000 & 4943 & 2.6665 & 0.2 & 0.9 \\
        sensor.1 & & max\_success & 462 & 1068 & 1.10018 & 0.05 & 0.99 \\
        wlan.2 & COL=2 & collisions & 28598 & 29661 & 1.04248 & 0.7 & 0.9 \\
        wlan4\_collide & COL=4,TIME\_MAX=10,k=4 & collide\_max & 345120 & 355727 & 1.06862565 & 0.1 & 0.9 \\
        wlan\_dl.0 & deadline=80 & deadline & 189703 & 254964 & 1.309220 & 0.9 & 0.9 \\
        zeroconf & reset=0,N=1000,K=1 & correct\_max & 31954 & 56230 & 1.2661 & 0.05 & 0.99 \\
    \end{tabular}
    \label{tab:mdp_smc_models}
\end{table}

\begin{sidewaystable}[p]
    \centering
    \setlength{\tabcolsep}{3pt}
    \caption{
        Number of samples required (averaged over ten runs) of each inference approach on all considered models.
        For readability, we give the number of samples divided by $100$, i.e.\ a cell value of $6$ means that between $600$ and $699$ samples were required.
    }
    \label{tab:app:mdp_smc_data}
    \sisetup{
        table-parse-only = true,
        table-alignment = right
    }
    \scriptsize
\newcommand{\rot}[1]{\rotatebox[origin=l]{90}{\texttt{#1}}}
\begin{tabular}{>{\ttfamily}rSSSSSSSSSSSSSSS}
            &   {\rot{bigmec}} & {\rot{consensus-mec}}   &   {\rot{consensus}} &   {\rot{csma}} & {\rot{evade-mdp}}   & {\rot{firewire\_dl}}   &   {\rot{manymecs}} & {\rot{mer}}   &   {\rot{pacman}} & {\rot{random-grid}}   & {\rot{sensor}}   &   {\rot{wlan}} & {\rot{wlan4\_collide}}   & {\rot{wlan\_dl}}   &   {\rot{zeroconf}}  \\
\midrule
 Bern-Bet         &              149 & 6971                    &                 333 &            297 & 5211                & 2673                  &               8011 & 2230          &              825 & 1760                  &             2377 &             57 & 3597                    & 541               &                 11 \\
 Bern-Exp         &              176 & 17980                   &                 757 &            949 & 12983               & 4443                  &               8474 & 4763          &             1094 & 5319                  &             4814 &            194 & 12420                   & 1751              &                 42 \\
 Bern-Sq          &              275 & 20364                   &                1001 &           1144 & 13962               & 6385                  &              13674 & 6618          &             1351 & 6051                  &             6388 &            209 & 12544                   & 1668              &                 41 \\
 Bern-Stitch      &               98 & 9440                    &                 387 &            406 & 6194                & 2037                  &               5246 & 2320          &              494 & 2672                  &             2559 &            100 & 6182                    & 905               &                 18 \\
 BetaTest         &               69 & 4244                    &                 186 &            198 & 13481               & 1233                  &               3782 & 1456          &              525 & 3062                  &             1368 &             28 & 6878                    & 412               &                 32 \\
 CAV19            &              400 & T/O                     &                4409 &           3897 & 100073              & T/O                   &              19407 & T/O           &              668 & T/O                   &          1938086 &          39310 & T/O                     & T/O               &                452 \\
 CP-Exp           &              167 & 6510                    &                 319 &            350 & 4961                & 2439                  &               8110 & 2529          &              805 & 1428                  &             1819 &             36 & 3355                    & 220               &                 15 \\
 CP-Sq            &              185 & 8149                    &                 462 &            488 & 5571                & 4066                  &               9311 & 3514          &              961 & 1845                  &             2816 &             37 & 3381                    & 195               &                 18 \\
 Global-Bern-Exp  &              151 & T/O                     &                4688 &            762 & 120512              & 72085                 &               9219 & T/O           &              851 & T/O                   &           841198 &            823 & 467417                  & T/O               &                 17 \\
 Global-Bern-Sq   &              152 & 585007                  &                1326 &            685 & 49189               & 5107                  &               7414 & 108419        &              643 & T/O                   &            12694 &            487 & 16831                   & T/O               &                 23 \\
 Global-CP-Exp    &              158 & T/O                     &                1031 &            358 & T/O                 & 42594                 &               7353 & T/O           &              423 & T/O                   &           868689 &            796 & 164039                  & T/O               &                  7 \\
 Global-CP-Sq     &               91 & 941457                  &                 743 &            273 & 44658               & 3120                  &               4653 & T/O           &              340 & T/O                   &             7988 &            285 & 6030                    & T/O               &                  8 \\
 Global-Hoeff-Exp &              151 & T/O                     &                1433 &            646 & 850852              & 58978                 &               6111 & T/O           &              509 & T/O                   &           722519 &            400 & 244134                  & T/O               &                 65 \\
 Global-Hoeff-Sq  &              100 & 824125                  &                 830 &            517 & 37892               & 3591                  &               6551 & 37064         &              476 & T/O                   &             9266 &            240 & 8407                    & T/O               &                 69 \\
 Hoeff-Bet        &              736 & 32321                   &                1692 &           1981 & 62667               & 14231                 &              40225 & 28405         &             4907 & 9351                  &            27207 &             85 & 15207                   & 668               &                702 \\
 Hoeff-Exp        &              168 & 8250                    &                 455 &            500 & 11799               & 3120                  &               8235 & 5142          &              829 & 2327                  &             4169 &             53 & 4880                    & 396               &                151 \\
 Hoeff-Sq         &              209 & 9597                    &                 583 &            741 & 13222               & 4524                  &              10561 & 6966          &             1233 & 2752                  &             5887 &             50 & 4895                    & 350               &                130 \\
 Hoeff-Stitch     &              119 & 7025                    &                 359 &            466 & 9034                & 2304                  &               5954 & 3909          &              667 & 1980                  &             3693 &             48 & 4056                    & 332               &                102 \\
 MLETest          &               65 & 4217                    &                 190 &            200 & 13431               & 1309                  &               3782 & 1395          &              533 & 3088                  &             1385 &             28 & 6872                    & 415               &                 35 \\
 ValueTest        &               76 & 4785                    &                 220 &            213 & 3173                & 1393                  &               4006 & 1564          &              468 & T/O                   &             1384 &             30 & 2811                    & 227               &                 12 \\
\end{tabular}
\end{sidewaystable}

\end{document}